\documentclass{article} % For LaTeX2e
\usepackage{iclr2021_conference,times}

\usepackage{amsmath,amsfonts,bm}

\def\eqref#1{equation~\ref{#1}}
\def\Eqref#1{Equation~\ref{#1}}
\def\1{\bm{1}}

\def\eps{{\epsilon}}

\def\vc{{\bm{c}}}

\def\vh{{\bm{h}}}

\def\vx{{\bm{x}}}

\def\vz{{\bm{z}}}

\def\mW{{\bm{W}}}

\DeclareMathAlphabet{\mathsfit}{\encodingdefault}{\sfdefault}{m}{sl}
\SetMathAlphabet{\mathsfit}{bold}{\encodingdefault}{\sfdefault}{bx}{n}

\newcommand{\KL}{D_{\mathrm{KL}}}
\newcommand{\Var}{\mathrm{Var}}

\usepackage[colorlinks]{hyperref}

\hypersetup{
    colorlinks = true,
    linkbordercolor = {white},
    citecolor={brown}
}

\usepackage{url}
\usepackage{booktabs}
\usepackage{multirow}
\usepackage{amsmath,amsfonts,amssymb}
\usepackage{tabularx}
\usepackage{graphicx}

\newtheorem{lemm}{Lemma}

\newtheorem{prop}{Proposition}
\newtheorem{defn}{Definition}
\newtheorem{assump}{Assumption}
\usepackage{makecell}
\usepackage{wrapfig}
\usepackage{caption}
\usepackage{pbox}
\usepackage{pifont}
\usepackage{placeins}
\usepackage{subcaption}
\usepackage{graphicx}
\usepackage{subcaption}

\newenvironment{proof}{\textit{Proof:}}{\hfill$\square$}

\newcommand{\cmark}{\ding{51}}%
\newcommand{\xmark}{\ding{55}}%

\title{Self-supervised Representation Learning with Relative Predictive Coding}

\author{%
  Yao-Hung Hubert Tsai$^1$, Martin Q. Ma$^1$, Muqiao Yang$^1$, \\
  {\bf Han Zhao}$^{23}$, {\bf  Louis-Philippe Morency}$^1$, {\bf Ruslan Salakhutdinov}$^1$ \\
 $^1$Carnegie Mellon University, $^2$D.E.\ Shaw \& Co., $^3$ University of Illinois at Urbana-Champaign
}

\iclrfinalcopy % Uncomment for camera-ready version, but NOT for submission.
\begin{document}

\maketitle

\begin{abstract}
This paper introduces Relative Predictive Coding (RPC), a new contrastive representation learning objective that maintains a good balance among training stability, minibatch size sensitivity, and downstream task performance. The key to the success of RPC is two-fold. First, RPC introduces the relative parameters to regularize the objective for boundedness and low variance. Second, RPC contains no logarithm and exponential score functions, which are the main cause of training instability in prior contrastive objectives. We empirically verify the effectiveness of RPC on benchmark vision and speech self-supervised learning tasks. Lastly, we relate RPC with mutual information (MI) estimation, showing RPC can be used to estimate MI with low variance \footnote{Project page: \url{https://github.com/martinmamql/relative_predictive_coding}}.

\end{abstract}

\section{Introduction}
 
Unsupervised learning has drawn tremendous attention recently because it can extract rich representations without label supervision. Self-supervised learning, a subset of unsupervised learning, learns representations by allowing the data to provide supervision \citep{devlin2018bert}. Among its mainstream strategies, self-supervised contrastive learning has been successful in visual object recognition \citep{he2020momentum, tian2019contrastive, chen2020big}, speech recognition \citep{oord2018representation, riviere2020unsupervised}, language modeling \citep{kong2019mutual}, graph representation learning \citep{velickovic2019deep} and reinforcement learning \citep{kipf2019contrastive}. The idea of self-supervised contrastive learning is to learn latent representations such that related instances (e.g., patches from the same image; defined as \textit{positive} pairs) will have representations within close distance, while unrelated instances (e.g., patches from two different images; defined as \textit{negative} pairs) will have distant representations \citep{arora2019theoretical}.

Prior work has formulated the contrastive learning objectives as maximizing the divergence between the distribution of related and unrelated instances. In this regard, different divergence measurement often leads to different loss function design. For example, variational mutual information (MI) estimation~\citep{poole2019variational} inspires Contrastive Predictive Coding (CPC)~\citep{oord2018representation}. Note that MI is also the KL-divergence between the distributions of related and unrelated instances~\citep{cover2012elements}. While the choices of the contrastive learning objectives are abundant~\citep{hjelm2018learning,poole2019variational,ozair2019wasserstein}, we point out that there are three challenges faced by existing methods. 

The first challenge is the training stability, where an unstable training process with high variance may be problematic. For example,~\citet{hjelm2018learning,tschannen2019mutual,tsai2020neural} show that the contrastive objectives with large variance cause numerical issues and have a poor downstream performance with their learned representations. The second challenge is the sensitivity to minibatch size, where the objectives requiring a huge minibatch size may restrict their practical usage. For instance, SimCLRv2~\citep{chen2020big} utilizes CPC as its contrastive objective and reaches state-of-the-art performances on multiple self-supervised and semi-supervised benchmarks.
Nonetheless, the objective is trained with a minibatch size of $8,192$, and this scale of training requires enormous computational power. The third challenge is the downstream task performance, which is the one that we would like to emphasize the most. For this reason, in most cases, CPC is the objective that we would adopt for contrastive representation learning, due to its favorable performance in downstream tasks~\citep{tschannen2019mutual,baevski2020wav2vec}.

This paper presents a new contrastive representation learning objective: the Relative Predictive Coding (\textit{RPC}), which attempts to achieve a good balance among these three challenges: training stability, sensitivity to minibatch size, and downstream task performance. At the core of RPC is the {\em relative parameters}, which are used to regularize RPC for its boundedness and low variance. From a modeling perspective, the relative parameters act as a $\ell_2$ regularization for RPC. From a statistical perspective, the relative parameters prevent RPC from growing to extreme values, as well as upper bound its variance. In addition to the relative parameters, RPC contains no logarithm and exponential, which are the main cause of the training instability for prior contrastive learning objectives~\citep{song2019understanding}.

To empirically verify the effectiveness of RPC, we consider benchmark self-supervised representation learning tasks, including visual object classification on CIFAR-10/-100~\citep{krizhevsky2009learning}, STL-10~\citep{coates2011analysis}, and ImageNet~\citep{russakovsky2015imagenet} and speech recognition on LibriSpeech~\citep{panayotov2015librispeech}. Comparing RPC to prior contrastive learning objectives, we observe a lower variance during training, a lower minibatch size sensitivity, and consistent performance improvement. Lastly, we also relate RPC with MI estimation, empirically showing that RPC can estimate MI with low variance.

%RPC has better training stability, because the objective as well as its variance can be bounded. Also, RPC adopts Chi-square divergence that avoids the logarithmic and exponential terms, which partially cause training instabilities in contrastive learning objectives \citep{song2019understanding,poole2019variational, hjelm2018learning, tschannen2019mutual}. In practice, we observe RPC enjoys much lower variance than prior methods such as DV \citep{donsker1975asymptotic} and NWJ \citep{nguyen2010estimating}. Meanwhile, RPC has less minibatch-size sensitivity than CPC \citep{oord2018representation}, which means that RPC could achieve more comparable downstreaem performance even when the minibatch size is limited. Lastly, RPC outperforms all previous methods on both visual object recognition and speech recognition tasks.

%We will illustrate the following in the experiments section:  1) Results from RPC and other objectives on downstream tasks. 2) Theoretical analysis and empirical demonstration of the variance of RPC. 3) Performance comparison under small minibatch sizes between RPC and CPC. 4) Effect of parameters in RPC and relation to mutual information estimation.

% under the \textit{unsupervised pretraining} followed by \textit{supervised fine-tuning} paradigm \citep{chen2020big} 

\vspace{-1mm}
% \han{Usually the caption of tables are required to place at top of the actual tables, so I moved the captions to the top.}
\section{Proposed Method}
% MOdeling benefits
% Cross entropy
\begin{table}[t!]
    \vspace{-4mm}
    \centering
    \caption{\small Different contrastive learning objectives, grouped by measurements of distribution divergence. $P_{XY}$ represents the distribution of related samples (positively-paired), and $P_XP_Y$ represents the distribution of unrelated samples (negatively-paired). 
    $f(x, y)\in\mathcal{F}$ for $\mathcal{F}$ being any class of functions $f:\mathcal{X}\times \mathcal{Y}\rightarrow \mathbb{R}$. 
    %$\dagger$: $J_{\rm JSD}$ has better stability than $J_{\rm DV}$ and $J_{\rm NWJ}$ but worse than $J_{\rm CPC}$, $J_{\rm WPC}$ and $J_{\rm RPC}$. 
    $\dagger$: Compared to $J_{\rm CPC}$ and $J_{\rm RPC}$, we empirically find $J_{\rm WPC}$ performs worse on complex real-world image datasets spanning CIFAR-10/-100~\citep{krizhevsky2009learning} and ImageNet~\citep{russakovsky2015imagenet}. 
    }
    \label{tab:objectives}
    \scalebox{0.59}{
    \begin{tabular}{c c c c}
         \toprule
         Objective & \makecell{ Good Training\\Stability} & \makecell{Lower Minibatch\\Size Sensitivity} & \makecell{Good Downstream\\Performance} \\
         \midrule \midrule 
          \multicolumn{4}{c}{relating to KL-divergence between $P_{XY}$ and $P_XP_Y$: 
          {\color{red} $J_{\rm DV}$}~\citep{donsker1975asymptotic}, {\color{red}$J_{\rm NWJ}$}~\citep{nguyen2010estimating}, and {\color{red}$J_{\rm CPC}$}~\citep{oord2018representation}} \\ \cmidrule{1-4}
         ${\color{red}J_{\rm DV}(X,Y)}:={\rm sup}_{f\in\mathcal{F}}\,\mathbb{E}_{P_{XY}}[f(x,y)]-\log (\mathbb{E}_{P_XP_Y}[e^{f(x,y)}])$ & \xmark & \cmark & \xmark \\ 
         ${\color{red}J_{\rm NWJ}(X,Y)}:={\rm sup}_{f\in\mathcal{F}}\,\mathbb{E}_{P_{XY}}[f(x,y)]- \mathbb{E}_{P_XP_Y}[e^{f(x,y) - 1}]$ & \xmark & \cmark & \xmark \\
                  ${\color{red}J_{\rm CPC}(X,Y)} := {\rm sup}_{f\in\mathcal{F}}\,\mathbb{E}_{(x, y_1) \sim P_{XY}, \{y_j\}_{j=2}^{N} \sim P_Y} \Big[ \log \big(e^{f(x, y_1)} / \frac{1}{N}\sum_{j=1}^N e^{f(x, y_j)}\big)\Big]$ & \cmark & \xmark & \cmark \\
         \midrule \midrule 
         \multicolumn{4}{c}{relating to JS-divergence between $P_{XY}$ and $P_XP_Y$: {\color{red}$J_{\rm JS}$}~\citep{nowozin2016f}} \\ \cmidrule{1-4}
         ${\color{red}J_{\rm JS}(X,Y)}:={\rm sup}_{f\in\mathcal{F}}\,\mathbb{E}_{P_{XY}}[-\log( 1 + e^{-f(x,y)})]-\mathbb{E}_{P_XP_Y}[\log( 1 + e^{f(x,y)})]$ & \cmark & \cmark & \xmark \\
         \midrule \midrule 
         \multicolumn{4}{c}{relating to Wasserstein-divergence between $P_{XY}$ and $P_XP_Y$: {\color{red}$J_{\rm WPC}$}~\citep{ozair2019wasserstein}, with $\mathcal{F_L}$ denoting the space of 1-Lipschitz functions} \\ \cmidrule{1-4}
         ${\color{red}J_{\rm WPC}(X,Y)}:={\rm sup}_{f\in\mathcal{F_L}}\,\mathbb{E}_{(x, y_1) \sim P_{XY}, \{y_j\}_{j=2}^{N} \sim P_Y} \Big[\log \big( e^{f(x, y_1)} / \frac{1}{N}\sum_{j=1}^N e^{f(x, y_j)}\big)\Big]$ & \cmark & \cmark & \xmark$^\dagger$\\ \midrule \midrule 
         \multicolumn{4}{c}{relating to $\chi^2$-divergence between $P_{XY}$ and $P_XP_Y$: {\color{red}$J_{\rm RPC}$} (ours)} \\ \cmidrule{1-4}
         ${\color{red}J_{\rm RPC}(X,Y)}:={\rm sup}_{f\in\mathcal{F}}\,\mathbb{E}_{P_{XY}}[f(x,y)]-\alpha \mathbb{E}_{P_XP_Y}[f(x,y)]-\frac{\beta}{2} \mathbb{E}_{P_{XY}}\left[f^{2}(x,y)\right]-\frac{\gamma}{2}\mathbb{E}_{P_XP_Y}\left[f^{2}(x,y)\right]$ & \cmark & \cmark & \cmark \\
         \bottomrule
    \end{tabular}
    }
    \vspace{-2mm}
    \vspace{-2mm}
\end{table}

This paper presents a new contrastive representation learning objective - the Relative Predictive Coding (RPC). At a high level, RPC 1) introduces the relative parameters to regularize the objective for boundedness and low variance; and 2) achieves a good balance among the three challenges in the contrastive representation learning objectives: training stability, sensitivity to minibatch size, and downstream task performance. We begin by describing prior contrastive objectives along with their limitations on the three challenges in Section~\ref{subsec:prelim}. Then, we detail our presented objective and its modeling benefits in Section~\ref{subsec:RPC}. An overview of different contrastive learning objectives is provided in Table~\ref{tab:objectives}. We defer all the proofs in Appendix.

\paragraph{Notation} 
We use an uppercase letter to denote a random variable (e.g., $X$), a lower case letter to denote the outcome of this random variable (e.g., $x$), and a calligraphy letter to denote the sample space of this random variable (e.g., $\mathcal{X}$). Next, if the samples $(x,y)$ are related (or positively-paired), we refer $(x,y)\sim P_{XY}$ with $P_{XY}$ being the joint distribution of $X\times Y$. If the samples $(x,y)$ are unrelated (negatively-paired), we refer $(x,y)\sim P_XP_Y$ with $P_XP_Y$ being the product of marginal distributions over $X\times Y$. Last, we define $f\in \mathcal{F}$ for $\mathcal{F}$ being any class of functions $f: \mathcal{X} \times \mathcal{Y} \to \mathbb{R}$.

\subsection{Preliminary}
\label{subsec:prelim}
% \subsection{Contrastive Learning Objectives}
Contrastive representation learning encourages the {\em contrastiveness} between the positive and the negative pairs of the representations from the related data $X$ and $Y$. Specifically, when sampling a pair of representations $(x,y)$ from their joint distribution ($(x,y)\sim P_{XY}$), this pair is defined as a positive pair; when sampling from the product of marginals ($(x,y)\sim P_XP_Y$), this pair is defined as a negative pair. Then,~\citet{tsai2020neural} formalizes this idea such that the contrastiveness of the representations can be measured by the divergence between $P_{XY}$ and $P_XP_Y$, where higher divergence suggests better contrastiveness. To better understand prior contrastive learning objectives, we categorize them in terms of different divergence measurements between $P_{XY}$ and $P_XP_Y$, with their detailed objectives presented in Table~\ref{tab:objectives}. 

% Contrastive representation learning encourages the {\em contrastiveness} of the representations between related and unrelated data. Specifically, given a pair of representations $(x,y)$, the contrastiveness tells whether this pair is a positive pair ($(x,y)\sim P_{XY}$) or a negative pair ($(x,y)\sim P_XP_Y$). Then,~\citet{tsai2020neural} formalizes this idea such that the contrastiveness of the representations can be measured by the divergence between $P_{XY}$ and $P_XP_Y$, where higher divergence suggests better contrastiveness. To better understand prior contrastive learning objectives, we categorize them in terms of different divergence measurements between $P_{XY}$ and $P_XP_Y$, with their detailed objectives presented in Table~\ref{tab:objectives}. 

We instantiate the discussion using Contrastive Predictive Coding \citep[$J_{\rm CPC}$]{oord2018representation}, which is a lower bound of $D_{\rm KL}(P_{XY}\,\|\,P_XP_Y)$ with $D_{\rm KL}$ referring to the KL-divergence: 
\begin{equation}
    J_{\rm CPC}(X,Y) := \underset{f\in \mathcal{F}}{\rm sup}\,\mathbb{E}_{(x, y_1) \sim P_{XY}, \{y_j\}_{j=2}^{N} \sim P_Y} \Big[ \log  \frac{e^{f(x, y_1)}}{ \frac{1}{N}\sum_{j=1}^N e^{f(x, y_j)}}\Big].
\label{eq:CPC}
\end{equation}
Then,~\citet{oord2018representation} presents to maximize $J_{\rm CPC}(X,Y)$, so that the learned representations $X$ and $Y$ have high contrastiveness. We note that $J_{\rm CPC}$ has been commonly used in many recent self-supervised representation learning frameworks~\citep{he2020momentum,chen2020simple}, where they constrain the function to be $f(x,y) =  {\rm cosine}(x,y)$ with ${\rm cosine}(\cdot)$ being cosine similarity. Under this function design, maximizing $J_{\rm CPC}$ 
leads the representations of related pairs to be close and representations of unrelated pairs to be distant.

The category of modeling $D_{\rm KL}(P_{XY}\,\|\,P_XP_Y)$ also includes the Donsker-Varadhan objective ($J_{\rm DV}$~\citep{donsker1975asymptotic,belghazi2018mine}) and the Nguyen-Wainright-Jordan objective ($J_{\rm NWJ}$~\citep{nguyen2010estimating,belghazi2018mine}), where~\citet{belghazi2018mine,tsai2020neural} show that $J_{\rm DV}(X,Y) = J_{\rm NWJ}(X,Y) = D_{\rm KL}(P_{XY}\,\|\,P_XP_Y)$.
The other divergence measurements considered in prior work are $D_{\rm JS}(P_{XY}\,\|\,P_XP_Y)$ (with $D_{\rm JS}$ referring to the Jenson-Shannon divergence) and $D_{\rm Wass}(P_{XY}\,\|\,P_XP_Y)$ (with $D_{\rm Wass}$ referring to the Wasserstein-divergence). The instance of modeling $D_{\rm JS}(P_{XY}\,\|\,P_XP_Y)$ is the Jensen-Shannon f-GAN objective \big($J_{\rm JS}$~\citep{nowozin2016f,hjelm2018learning}\big), where $J_{\rm JS}(X,Y)=2\big(D_{\rm JS}(P_{XY}\,\|\,P_XP_Y) - \log 2\big)$.\footnote{$J_{\rm JS}(X,Y)$ achieves its supreme value when $f^*(x,y)=\log (p(x,y) / p(x)p(y))$~\citep{tsai2020neural}. Plug-in $f^*(x,y)$ into $J_{\rm JS}(X,Y)$, we can conclude $J_{\rm JS}(X,Y)=2(D_{\rm JS}(P_{XY}\,\|\,P_XP_Y) - \log 2)$.} The instance of modeling $D_{\rm Wass}(P_{XY}\,\|\,P_XP_Y)$ is the Wasserstein Predictive Coding \big($J_{\rm WPC}$~\citep{ozair2019wasserstein}\big), where $J_{\rm WPC}(X,Y)$ 
modifies $J_{\rm CPC}(X,Y)$ objective (\eqref{eq:CPC}) by searching the function from $\mathcal{F}$ to $\mathcal{F_L}$. $\mathcal{F_L}$ denotes any class of 1-Lipschitz continuous functions from $(\mathcal{X} \times \mathcal{Y})$ to $\mathbb{R}$, and thus $\mathcal{F_L} \subset \mathcal{F}$. \citet{ozair2019wasserstein} shows that $J_{\rm WPC}(X,Y)$ is the lower bound of both $D_{\rm KL}(P_{XY}\,\|\,P_XP_Y)$ and $D_{\rm Wass}(P_{XY}\,\|\,P_XP_Y)$. See Table~\ref{tab:objectives} for all the equations. To conclude, the contrastive representation learning objectives are unsupervised representation learning methods that maximize the distribution divergence between $P_{XY}$ and $P_XP_Y$. The learned representations cause high contrastiveness, and recent work~\citep{arora2019theoretical,tsai2020demystifying} theoretically show that highly-contrastive representations could improve the performance on downstream tasks.

After discussing prior contrastive representation learning objectives, we point out three challenges in their practical deployments: training stability, sensitivity to minibatch training size, and downstream task performance. In particular, the three challenges can hardly be handled well at the same time, where we highlight the conclusions in Table~\ref{tab:objectives}. {\bf \em Training Stability:} The training stability highly relates to the variance of the objectives, where~\citet{song2019understanding} shows that $J_{\rm DV}$ and $J_{\rm NWJ}$ exhibit inevitable high variance due to their inclusion of exponential function. As pointed out by~\citet{tsai2020neural}, $J_{\rm CPC}$, $J_{\rm WPC}$, and $J_{\rm JS}$ have better training stability because $J_{\rm CPC}$ and $J_{\rm WPC}$ can be realized as a multi-class classification task and $J_{\rm JS}$ can be realized as a binary classification task. The cross-entropy loss adopted in $J_{\rm CPC}$, $J_{\rm WPC}$, and $J_{\rm JS}$ is highly-optimized and stable in existing optimization package~\citep{abadi2016tensorflow,paszke2019pytorch}. {\bf \em Sensitivity to minibatch training size:} Among all the prior contrastive representation learning methods, $J_{\rm CPC}$ is known to be sensitive to the minibatch training size \citep{ozair2019wasserstein}. Taking a closer look at~\eqref{eq:CPC}, $J_{\rm CPC}$ deploys an instance selection such that
%, for $(x, y_1) \sim P_{XY}, \{y_j\}_{j=2}^{N} \sim P_Y^{\otimes N-1}$, 
$y_1$ should be selected from $\{y_1, y_2, \cdots, y_N\}$, with $(x, y_1) \sim P_{XY}$, $(x, y_{j > 1}) \sim P_{X}P_{Y}$ with $N$ being the minibatch size.
%that considered in $J_{\rm CPC}$, and
Previous work \citep{poole2019variational,song2019understanding, chen2020simple,caron2020unsupervised} showed that a large $N$ results in a more challenging instance selection and forces $J_{\rm CPC}$ to have a better contrastiveness of $y_1$ (related instance for $x$) against $\{y_j\}_{j=2}^{N}$ (unrelated instance for $x$). $J_{\rm DV}$, $J_{\rm NWJ}$, and $J_{\rm JS}$ do not consider the instance selection, and $J_{\rm WPC}$ reduces the minibatch training size sensitivity by enforcing 1-Lipschitz constraint. {\bf \em Downstream Task Performance:} The downstream task performance is what we care the most among all the three challenges.  $J_{\rm CPC}$ has been the most popular objective as it manifests superior performance over the other alternatives~\citep{tschannen2019mutual,tsai2020neural,tsai2020demystifying}. We note that although $J_{\rm WPC}$ shows better performance on Omniglot \citep{lake2015human} and CelebA \citep{liu2015deep} datasets, we empirically find it not generalizing well to CIFAR-10/-100~\citep{krizhevsky2009learning} and ImageNet~\citep{russakovsky2015imagenet}. %We \colorbox{yellow}{argue} the reason may because that the 1-Lipschitz constraint is too strong to let $J_{\rm WPC}$ learn representations with high contrastiveness. 

\subsection{Relative Predictive Coding}
\label{subsec:RPC}

%Until now, we describe the three challenges of prior contrastive learning objectives: training stability, mini-batch size sensitivity, and downstream performance. 
In this paper, we present Relative Predictive Coding (RPC), which achieves a good balance among the three challenges mentioned above:
\begin{equation} 
    J_{\rm RPC}(X,Y):=\underset{f\in\mathcal{F}}{\rm sup}\,\mathbb{E}_{P_{XY}}[f(x,y)]-\alpha \mathbb{E}_{P_XP_Y}[f(x,y)]-\frac{\beta}{2} \mathbb{E}_{P_{XY}}\left[f^{2}(x,y)\right]-\frac{\gamma}{2}\mathbb{E}_{P_XP_Y}\left[f^{2}(x,y)\right],
    % I = \,\,\frac{1}{|\sP|}\sum_{(\tilde{\vx_1}, \tilde{\vx_2}) \in \sP}f - \frac{\alpha}{|\sN|}\, \sum_{(\tilde{\vx_1}, \tilde{\vx_2}) \in \sN}f - \frac{\beta}{2|\sP|}\, \sum_{(\tilde{\vx_1}, \tilde{\vx_2}) \in \sP}f^2 - \frac{\gamma}{2|\sN|}\,\sum_{(\tilde{\vx_1}, \tilde{\vx_2} \in \sN)}f^2
\label{eq:RPC_math_def}
\end{equation}
where $\alpha > 0$, $\beta > 0$, $\gamma >0$ are hyper-parameters and we define them as {\em relative parameters}. Intuitively, $J_{\rm RPC}$ contains no logarithm or exponential, potentially preventing unstable training due to numerical issues. Now, we discuss the roles of $\alpha, \beta, \gamma$. At a first glance, $\alpha$ acts to discourage the scores of $P_{XY}$ and $P_XP_Y$ from being close, and $\beta/\gamma$ acts as a $\ell_2$ regularization coefficient to stop $f$ from becoming large. For a deeper analysis, the relative parameters act to regularize our objective for boundedness and low variance. To show this claim, we first present the following lemma:
\vspace{-1mm}
\begin{lemm}[Optimal Solution for $J_{\rm RPC}$] Let $r(x,y) = \frac{p(x,y)}{p(x)p(y)}$ be the density ratio. $J_{\rm RPC}$ has the optimal solution
$f^*(x,y) = \frac{r(x,y) - \alpha}{\beta \, r(x,y) + \gamma}:=r_{\alpha, \beta, \gamma}(x,y)$ with $-\frac{\alpha}{\gamma}\leq r_{\alpha, \beta, \gamma} \leq \frac{1}{\beta}$.
%\,\,{\rm with}\,\, -\frac{\alpha}{\gamma} \leq f^*(x,y) \leq \frac{1}{\beta}.
\label{lemma:ratio}
\end{lemm}
\vspace{-1mm}
Lemma~\ref{lemma:ratio} suggests that $J_{\rm RPC}$ achieves its supreme value at the ratio $r_{\alpha, \beta, \gamma}(x,y)$ indexed by the relative parameters $\alpha, \beta, \gamma$ (i.e., we term $r_{\alpha, \beta, \gamma}(x,y)$ as the relative density ratio). We note that $r_{\alpha, \beta, \gamma}(x,y)$ is an increasing function w.r.t. $r(x,y)$ and is nicely bounded even when $r(x,y)$ is large. We will now show that the bounded $r_{\alpha, \beta, \gamma}$ suggests
% that, when empirically estimate $J_{\rm RPC}$, it is safe to use a bounded function for its estimation without losing precision. 
the empirical estimation of $J_{\rm RPC}$ has boundeness and low variance. In particular, let $\{x_i, y_i\}_{i=1}^n$ be $n$ samples drawn uniformly at random from $P_{XY}$ and $\{x'_j, y'_j\}_{j=1}^m$ be $m$ samples drawn uniformly at random from $P_XP_Y$. 
%\begin{equation} \label{eq:RPC_empirical_def}
%\hat{I}_{\mathrm{RPC}}^{m, n}=\mathbb{E}_{\hat{P}_{m}}[f_{\theta}]-\alpha \mathbb{E}_{\hat{Q}_{n}}[f_{\theta}]-\frac{\beta}{2} \mathbb{E}_{\hat{P}_{m}}\left[f_{\theta}^{2}\right]-\frac{\gamma}{2}\mathbb{E}_{\hat{Q}_{n}}\left[f_{\theta}^{2}\right].
%\end{equation}
Then, we use neural networks
%(with bounded values) 
to empirically estimate $J_{\rm RPC}$ as $\hat{J}^{m,n}_{\rm RPC}$:
%and show its boundedness and low variance:
\begin{defn}[$\hat{J}^{m,n}_{\rm RPC}$, empirical estimation of $J_{\rm RPC}$] We parametrize $f$ via a family of neural networks $\mathcal{F}_\Theta:= \{f_\theta: \theta\in\Theta\subseteq\mathbb{R}^d
%, -\frac{\alpha}{\gamma}\leq f_\theta \leq \frac{1}{\beta}
\}$ where $d \in\mathbb{N}$ and $\Theta$ is compact. Then,
$\hat{J}^{m,n}_{\rm RPC} = \sup_{{f}_\theta \in \mathcal{F}_\Theta} \frac{1}{n}\sum_{i=1}^n f_\theta(x_i,y_i) -  \frac{1}{m}\sum_{j=1}^m \alpha f_\theta(x'_j,y'_j) - \frac{1}{n}\sum_{i=1}^n  \frac{\beta}{2} f_\theta^2(x_i,y_i) -  \frac{1}{m}\sum_{j=1}^m \frac{\gamma}{2} f_\theta^2(x'_j,y'_j).
$
\label{def:emp_obj}
\end{defn}
\vspace{-4mm}
\begin{prop}[Boundedness of $\hat{J}^{m,n}_{\rm RPC}$, informal] $0 \leq J_{\rm RPC} \leq \frac{1}{2\beta}+\frac{\alpha^2}{2\gamma}$. Then, with probability at least $1-\delta$, $|J_{\rm RPC} - \hat{J}^{m,n}_{\rm RPC}| =  O(\sqrt{\frac{d + {\rm log}\,(1/\delta)}{n'}}),$ where $n' = {\rm min}\,\{n, m\}$. 
\label{prop:RPC_bound}
\end{prop}
\vspace{-3mm}
%\begin{prop}[$J_{\rm RPC}$ as Relative Chi-square Divergence] With $f^*(x,y) = r_{\alpha, \beta, \gamma}(x,y)$, $J_{\rm RPC}$ has
%$J_{RPC}(X,Y) = \frac{1}{2}\mathbb{E}_{P_{XY}}[r_{\alpha, \beta, \gamma}(x,y)] - \frac{\alpha}{2}\mathbb{E}_{P_XP_Y}[r_{\alpha, \beta, \gamma}(x,y)]$ with $ max\{0, 2-2\alpha -\beta -\gamma\}\leq J_{RPC} \leq \frac{1}{2\beta}+\frac{\alpha^2}{2\gamma}$.
%\label{prop:RPC}
%\end{prop}
%By definition, $D_{\chi^2}(P_{XY}\,\|\,P_XP_Y) = \frac{1}{2}\mathbb{E}_{P_{XY}}[r(x,y)]-\frac{1}{2}$. Proposition~\ref{prop:RPC} suggests
%$D_{\chi^2}(P_{XY}\,\|\,P_XP_Y) = J_{RPC}|_{\alpha=0, \beta=0, \gamma=1} - \frac{1}{2}$. 
%\paragraph{Remark on Practical Deployment}
\begin{prop}[Variance of $\hat{J}^{m,n}_{\rm RPC}$, informal] \label{prop:variance_informal} There exist universal constants $c_1$ and $c_2$ that depend only on $\alpha, \beta, \gamma$, such that $
%\Var [\hat{J}^{m,n}_{{\rm RPC}}] \leq \frac{1}{n} {\rm max}\bigg\{\Big(\frac{2\alpha \gamma + \beta \alpha^2}{2 \gamma^2}\Big)^2, \Big(\frac{1}{2\beta}\Big)^2\bigg\} + \frac{1}{m} {\rm max}\bigg\{\Big(\frac{\alpha^2 }{2 \gamma}\Big)^2, \Big( \frac{2\alpha \beta + \gamma}{2\beta^2} \Big)^2\bigg\}.
\Var [\hat{J}^{m,n}_{{\rm RPC}}] = O \Big(\frac{c_1}{n} + \frac{c_2}{m} \Big).
$
\end{prop}

From the two propositions, when $m$ and $n$ are large, i.e., the sample sizes are large, $\hat{J}^{m,n}_{\rm RPC}$ is bounded, and its variance vanishes to $0$. First, the boundedness of $\hat{J}^{m,n}_{\rm RPC}$ suggests $\hat{J}^{m,n}_{\rm RPC}$ will not grow to extremely large or small values. Prior contrastive learning objectives with good training stability (e.g., $J_{\rm CPC}$/$J_{\rm JS}$/$J_{\rm WPC}$) also have the boundedness of their objective values. For instance, the empirical estimation of ${J}_{\rm CPC}$ is less than ${\rm log}\,N$ (\eqref{eq:CPC})~\citep{poole2019variational}. Nevertheless, ${J}_{\rm CPC}$ often performs the best only when minibatch size is large, and empirical performances of $J_{\rm JS}$ and $J_{\rm WPC}$ are not as competitive as ${J}_{\rm CPC}$. Second, the upper bound of the variance implies the training of $\hat{J}^{m,n}_{\rm RPC}$ can be stable, and in practice we observe a much smaller value than the stated upper bound. On the contrary, \citet{song2019understanding} shows that the empirical estimations of ${J}_{\rm DV}$ and ${J}_{\rm NWJ}$ exhibit inevitable variances that grow exponentially with the true $\KL(P_{XY} \Vert P_XP_Y)$. 

Lastly, similar to prior contrastive learning objective that are related to distribution divergence measurement, we associate $J_{\rm RPC}$ with the Chi-square divergence $D_{\chi^2}(P_{XY}\,\|\,P_XP_Y) = \mathbb{E}_{P_XP_Y}[r^2(x,y)] - 1$~\citep{nielsen2013chi}. The derivations are provided in Appendix. By having $P' = \frac{\beta}{\beta+\gamma}P_{XY} + \frac{\gamma}{\beta+\gamma}P_XP_Y$ as the mixture distribution of $P_{XY}$ and $P_XP_Y$, we can rewrite $J_{\rm RPC}(X,Y)$ as $J_{\rm RPC}(X,Y) = \frac{\beta + \gamma}{2} \mathbb{E}_{P'}[r_{\alpha, \beta, \gamma}^2(x,y)]$. Hence, $J_{\rm RPC}$ can be regarded as a generalization of $D_{\chi^2}$ with the relative parameters $\alpha, \beta, \gamma$, where $D_{\chi^2}$ can be recovered from $J_{\rm RPC}$ by specializing $\alpha=0$, $\beta=0$ and $\gamma=1$ (e.g., $D_{\chi^2} = 2J_{\rm RPC}|_{\alpha=
\beta=0, \gamma=1} - 1$). Note that $J_{\rm RPC}$ may not be a formal divergence measure with arbitrary $\alpha, \beta , \gamma$.

\section{Experiments} \label{sec:experiment}
We provide an overview of the experimental section. First, we conduct benchmark self-supervised representation learning tasks spanning visual object classification and speech recognition. This set of experiments are designed to discuss the three challenges of the contrastive representation learning objectives: downstream task performance (Section~\ref{subsec:downstream}), training stability (Section~\ref{subsec:stability}), and minibatch size sensitivity (Section~\ref{subsec:minibatch}). We also provide an ablation study on the choices of the relative parameters in $J_{\rm RPC}$ (Section~\ref{subsec: relative_parameter}). On these experiments we found that $J_{\rm RPC}$ achieves a lower variance during training, a lower batch size insensitivity, and consistent performance improvement. Second, we relate $J_{\rm RPC}$ with mutual information (MI) estimation (Section~\ref{subsec:MI}). The connection is that MI is an average statistic of the density ratio, and we have shown that the optimal solution of $J_{\rm RPC}$ is the relative density ratio (see Lemma~\ref{lemma:ratio}). Thus we could estimate MI using the density ratio transformed from the optimal solution of $J_{\rm RPC}$. On these two sets of experiments, we fairly compare $J_{\rm RPC}$ with other contrastive learning objectives. Particularly, across different objectives, we fix the network, learning rate, optimizer, and batch size (we use the default configurations suggested by the original implementations from \citet{chen2020big}, \citet{riviere2020unsupervised} and \citet{tsai2020neural}.) The only difference will be the objective itself. In what follows, we perform the first set of experiments. We defer experimental details in the Appendix.

\vspace{-2mm}
\paragraph{Datasets.} For the visual objective classification, we consider CIFAR-10/-100 \citep{krizhevsky2009learning}, STL-10 \citep{coates2011analysis}, and ImageNet \citep{russakovsky2015imagenet}. CIFAR-10/-100 and ImageNet contain labeled images only, while STL-10 contains labeled and unlabeled images. For the speech recognition, we consider LibriSpeech-100h~\citep{panayotov2015librispeech} dataset, which contains $100$ hours of $16{\rm kHz}$ English speech from $251$ speakers with $41$ types of phonemes. 

\vspace{-2mm}
\paragraph{Training and Evaluation Details.} For the vision experiments, we follow the setup from SimCLRv2~\citep{chen2020big}, which considers visual object recognition as its downstream task. For the speech experiments, we follow the setup from prior work~\citep{oord2018representation, riviere2020unsupervised}, which consider phoneme classification and speaker identification as the downstream tasks. Then, we briefly discuss the training and evaluation details into three modules: 1) related and unrelated data construction, 2) pre-training, and 3) fine-tuning and evaluation. For more details, please refer to Appendix or the original implementations. 

\vspace{-1mm}
\underline{\em $\triangleright\,\,$ Related and Unrelated Data Construction.} In the vision experiment, we construct the related images by applying different augmentations on the same image. Hence, when $(x,y)\sim P_{XY}$, $x$ and $y$ are the same image with different augmentations. The unrelated images are two randomly selected samples. In the speech experiment, we define the current latent feature (feature at time $t$) and the future samples (samples at time $>t$) as related data. In other words, the feature in the latent space should contain information that can be used to infer future time steps. A latent feature and randomly selected samples would be considered as unrelated data.

\vspace{-1mm}
\underline{\em $\triangleright\,\,$ Pre-training.} The pre-training stage refers to the self-supervised training by a contrastive learning objective. Our training objective is defined in Definition~\ref{def:emp_obj}, where we use neural networks to parametrize the function using the constructed related and unrelated data. Convolutional neural networks are used for vision experiments.  Transformers~\citep{vaswani2017attention} and LSTMs~\citep{hochreiter1997long} are used for speech experiments.

\vspace{-1mm}
\underline{\em $\triangleright\,\,$ Fine-tuning and Evaluation.} After the pre-training stage, we fix the parameters in the pre-trained networks and add a small fine-tuning network on top of them. Then, we fine-tune this small network with the downstream labels in the data's training split. For the fine-tuning network, both vision and speech experiments consider multi-layer perceptrons. Last, we evaluate the fine-tuned representations on the data's test split. We would like to point out that we do not normalize the hidden representations encoded by the pre-training neural network for loss calculation. This hidden normalization technique is widely applied \citep{tian2019contrastive, chen2020simple, chen2020big} to stabilize training and increase performance for prior objectives, but we find it unnecessary in $J_{\rm RPC}$.

\subsection{Downstream Task Performances on Vision and Speech} \label{subsec:downstream}

\begin{table}[t!]
\vspace{-4mm}
\caption{\small Top-1 accuracy (\%) for visual object recognition results. $J_{\rm DV}$ and $J_{\rm NWJ}$ are not reported on ImageNet due to numerical instability. ResNet depth, width and Selective Kernel (SK) configuration for each setting are provided in ResNet depth+width+SK column. A slight drop of $J_{\rm CPC}$ performance compared to \citet{chen2020big} is because we only train for $100$ epochs rather than $800$ due to the fact that running 800 epochs uninterruptedly on cloud TPU is very expensive. Also, we did not employ a memory buffer \citep{he2020momentum} to store negative samples. We and we did not employ a memory buffer. We also provide the results from fully supervised models as a comparison \citep{chen2020simple, chen2020big}. Fully supervised training performs worse on STL-10 because it does not employ the unlabeled samples in the dataset \citep{lowe2019putting}.}
\label{table:vision_res}
\begin{center}
\vspace{-4mm} 
\scalebox{0.88}{
\begin{tabular}{c c | c c c c c c | c}
    \toprule
    \multirow{2}{*}{Dataset} & \multirow{2}{*}{ResNet Depth+Width+SK} & \multicolumn{6}{c|}{Self-supervised} & \multirow{2}{*}{Supervised}\\
     & & $J_{\rm DV}$ & $J_{\rm NWJ}$ & $J_{\rm JS}$ & $J_{\rm WPC}$ & $J_{\rm CPC}$ & $J_{\rm RPC}$ & \\
    \midrule
    CIFAR-10 & 18 + 1$\times$ + No SK & 91.10 & 90.54 & 83.55 & 80.02 & 91.12 & \textbf{91.46} & 93.12\\
    CIFAR-10 & 50 + 1$\times$ + No SK & 92.23 & 92.67 & 87.34 & 85.93 & 93.42 & \textbf{93.57} & 95.70\\
    CIFAR-100 & 18 + 1$\times$ + No SK & 77.10 & 77.27 & 74.02 & 72.16 & 77.36 & \textbf{77.98} & 79.11\\
    CIFAR-100 & 50 + 1$\times$ + No SK & 79.02 & 78.52 & 75.31 & 73.23 & 79.31 & \textbf{79.89} & 81.20\\
    STL-10 & 50 + 1$\times$ + No SK & 82.25 & 81.17 & 79.07 & 76.50 & 83.40 & \textbf{84.10} & 71.40 \\
    ImageNet & 50 + 1$\times$ + SK & - & - & 66.21 & 62.10 & 73.48 & \textbf{74.43} & 78.50\\
    ImageNet & 152 + 2$\times$ + SK & - & - & 71.12 & 69.51 & 77.80 & \textbf{78.40} & 80.40\\
    \bottomrule
\end{tabular}}
\end{center}
\vspace{-3mm}
\end{table}

For the downstream task performance in the vision domain, we test the proposed $J_{\rm RPC}$ and other contrastive learning objectives on CIFAR-10/-100 \citep{krizhevsky2009learning}, STL-10 \citep{coates2011analysis}, and ImageNet ILSVRC-2012 \citep{russakovsky2015imagenet}. Here we report the best performances $J_{\rm RPC}$ can get on each dataset (we include experimental details in \ref{appendix:vision}.) Table \ref{table:vision_res} shows that the proposed $J_{\rm RPC}$ outperforms other objectives on all datasets. Using $J_{\rm RPC}$ on the largest network (ResNet with depth of $152$, channel width of $2$ and selective kernels), the performance jumps from $77.80\%$ of $J_{\rm CPC}$ to $78.40\%$ of $J_{\rm RPC}$.

Regarding speech representation learning, the downstream performance for phoneme and speaker classification are shown in Table \ref{tab:speech_res} (we defer experimental details in Appendix \ref{appendix:speech}.) Compared to $J_{\rm CPC}$, $J_{\rm RPC}$ improves the phoneme classification results with $4.8$ percent and the speaker classification results with $0.3$ percent, which is closer to the fully supervised model. Overall, the proposed $J_{\rm RPC}$ performs better than other unsupervised learning objectives on both phoneme classification and speaker classification tasks. 
% Following \citep{panayotov2015librispeech}, we use a 1-layer transformer network. 

\begin{table}[t!]
\caption{\small Accuracy (\%) for LibriSpeech-100h phoneme and speaker classification results. We also provide the results from fully supervised model as a comparison \citep{oord2018representation}. }
\label{tab:speech_res}
\vspace{-4mm}
\small
    \begin{center}
        \begin{tabular}{c | cccc | c}
            \toprule
            \multirow{2}{*}{Task Name} & \multicolumn{4}{c|}{Self-supervised} & \multirow{2}{*}{Supervised} \\ 
            & $J_{\rm CPC}$ & $J_{\rm DV}$ & $J_{\rm NWJ}$ & $J_{\rm RPC}$ & \\
            \midrule
            Phoneme classification & 64.6 & 61.27 & 62.09 & \textbf{69.39} & 74.6 \\
            Speaker classification & 97.4 & 95.36 & 95.89 & \textbf{97.68} & 98.5 \\
            \bottomrule
        \end{tabular}
    \end{center}
\vspace{-3mm}
\end{table}

% Although the difference of the performances between CPC and RPC is relatively small on speaker classification, RPC is closer to the supervised learning setting than phoneme classification. We will demonstrate that RPC could achieve not only a better, but also a more robust performance than CPC method in Section \ref{sec:robust}.

\begin{figure}[t!]
\centering
\includegraphics[width=\textwidth]{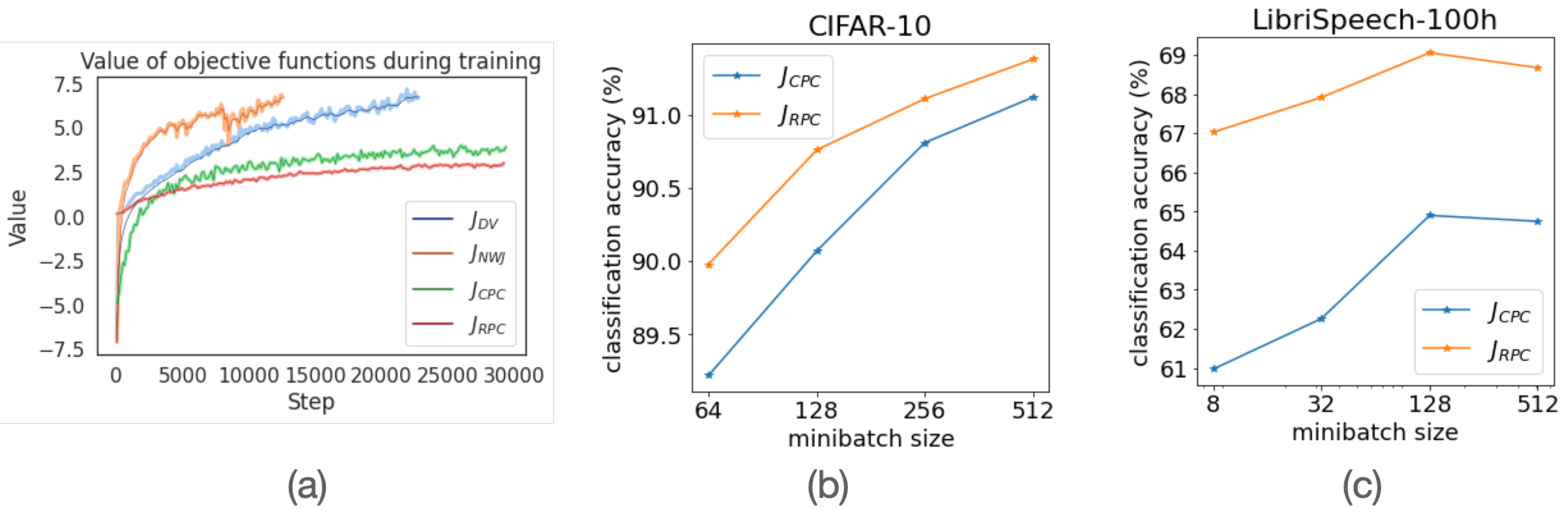}
\vspace{-7mm}
\caption{\small (a) Empirical values of $J_{\rm DV}$, $J_{\rm NWJ}$, $J_{\rm CPC}$ and $J_{\rm RPC}$ performing visual object recognition on CIFAR-10. $J_{\rm DV}$ and $J_{\rm NWJ}$ soon explode to NaN values and stop the training (shown as early stopping in the figure), while $J_{\rm CPC}$ and $J_{\rm RPC}$ are more stable. Performance comparison of $J_{\rm CPC}$ and $J_{\rm RPC}$ on (b) CIFAR-10 and (c) LibriSpeech-100h with different minibatch sizes, showing that the performance of $J_{\rm RPC}$ is less sensitive to minibatch size change compared to $J_{\rm CPC}$.}
\label{fig:var_batch}
\end{figure}
\subsection{Training Stability}
\label{subsec:stability}

We provide empirical training stability comparisons on $J_{\rm DV}$, $J_{\rm NWJ}$, $J_{\rm CPC}$ and $J_{\rm RPC}$ by plotting the values of the objectives as the training step increases. We apply the four objectives to the SimCLRv2 framework and train on the CIFAR-10 dataset. All setups of training are exactly the same except the objectives. From our experiments, $J_{\rm DV}$ and $J_{\rm NWJ}$ soon explode to NaN and disrupt training (shown as early stopping in Figure \ref{fig:var_batch}a; extremely large values are not plotted due to scale constraints). On the other hand, $J_{\rm RPC}$ and $J_{\rm CPC}$ has low variance, and both enjoy stable training. As a result, performances using the representation learned from unstable $J_{\rm DV}$ and $J_{\rm NWJ}$ suffer in downstream task, while representation learned by $J_{\rm RPC}$ and $J_{\rm CPC}$ work much better.

% \begin{figure}[ht]
% \begin{minipage}[c]{0.48\linewidth}
% \centering
% \includegraphics[scale=0.4]{fig/variance.png}
% \caption{\small Empirical values of $J_{\rm DV}$, $J_{\rm NWJ}$, $J_{\rm CPC}$ and $J_{\rm RPC}$ performing visual objective recognition on CIFAR-10. $J_{\rm DV}$ and $J_{\rm NWJ}$ soon explode to NaN values and stop training (shown as early stopping in the figure), while $J_{\rm CPC}$ and $J_{\rm RPC}$ are more stable.}
% \label{fig:variance}
% \end{minipage}
% \hspace{0.4cm}
% \begin{minipage}[c]{0.48\linewidth}
% \centering
% \includegraphics[height=30mm]{fig/batchsize_cifar.png}
% \includegraphics[height=30mm]{fig/batchsize_librispeech.png}
% \captionof{figure}{\small Performance comparison of $J_{\rm CPC}$ and $J_{\rm RPC}$ on CIFAR-10 and LibriSpeech-100h with different minibatch sizes, showing that the performance of $J_{\rm RPC}$ is less sensitive to minibatch size change compared to $J_{\rm CPC}$.}
% \label{fig:batchsize}
% \end{minipage}
% \end{figure}

\subsection{Minibatch Size Sensitivity}
\label{subsec:minibatch}

We then provide the analysis on the effect of minibatch size on $J_{\rm RPC}$ and $J_{\rm CPC}$, since $J_{\rm CPC}$ is known to be sensitive to minibatch size \citep{poole2019variational}. We train SimCLRv2 \citep{chen2020big} on CIFAR-10 and the model from \citet{riviere2020unsupervised} on LibriSpeech-100h using $J_{\rm RPC}$ and $J_{\rm CPC}$ with different minibatch sizes. The settings of relative parameters are the same as Section \ref{subsec:stability}. From Figure \ref{fig:var_batch}b and \ref{fig:var_batch}c, we can observe that both $J_{\rm RPC}$ and $J_{\rm CPC}$ achieve their optimal performance at a large minibatch size. However, when the minibatch size decreases, the performance of $J_{\rm CPC}$ shows higher sensitivity and suffers more when the number of minibatch samples is small. The result suggests that the proposed method might be less sensitive to the change of minibatch size compared to $J_{\rm CPC}$ given the same training settings.

% when the minibatch size is small, $J_{\rm RPC}$ has better performance than $J_{\rm CPC}$. When minibatch size increases, the performance of $J_{\rm RPC}$ and $J_{\rm CPC}$ gets closer, but $J_{\rm RPC}$ is still more robust to the change of minibatch size and shows more consistent performance on the downstream task on both datasets. 

% \begin{figure}[h]
% \begin{center}
%\framebox[4.0in]{$\;$}
% \includegraphics[scale=0.4]{fig/batchsize_cifar.png}
% \includegraphics[scale=0.4]{fig/batchsize_librispeech.png}
% \end{center}
% \caption{Performance comparison of CPC and RPC on CIFAR-10 and LibriSpeech-100h, when minibatch size is 8, 32, 128 and 512.}
% \label{fig:batchsize}
% \end{figure}

\subsection{Effect of Relative Parameters} \label{subsec: relative_parameter}
We study the effect of different combinations of relative parameters in $J_{\rm RPC}$ by comparing downstream performances on visual object recognition. We train SimCLRv2 on CIFAR-10 with different combinations of $\alpha, \beta$ and $\gamma$ in $J_{\rm RPC}$ and fix all other experimental settings. We choose $\alpha \in \{0, 0.001,1.0\}, \beta \in \{0, 0.001,1.0\}, \gamma \in \{0, 0.001,1.0\}$ and we report the best performances under each combination of $\alpha, \beta$, and $\gamma$. From Figure \ref{fig:relative_parameters_ablation}, we first observe that $\alpha > 0$ has better downstream performance than $\alpha=0$ when $\beta$ and $\gamma$ are fixed. This observation is as expected, since $\alpha >0$ encourages representations of related and unrelated samples to be pushed away. Then, we find that a small but nonzero $\beta$ ($\beta = 0.001$) and a large $\gamma$ ($\gamma = 1.0$) give the best performance compared to other combinations. Since $\beta$ and $\gamma$ serve as the coefficients of $\ell_2$ regularization, the results imply that the regularization is a strong and sensitive factor that will influence the performance. The results here are not as competitive as Table \ref{table:vision_res} because the CIFAR-10 result reported in Table \ref{table:vision_res} is using a set of relative parameters ($\alpha = 1.0, \beta = 0.005, \gamma = 1.0$) that is different from the combinations in this subsection. Also, we use quite different ranges of $\gamma$ on ImageNet (see \ref{appendix:vision} for details.) In conclusion, we find empirically that a non-zero $\alpha$, a small $\beta$ and a large $\gamma$ will lead to the optimal representation for the downstream task on CIFAR-10.  

% $\alpha=0$ influences less to the downstream performance unless being equal to $0$. The major reason is that in the proposed $J_{\rm RPC}$, $\beta$ and $\gamma$ serve as regularization coefficients. $\alpha$ is less effective because it only regularizes a linear function of the negative terms, thus the gradients during backpropagation will be less dominant than the gradients from the quadratic term regularized by $\beta$. 

% {\color{red} We also study the effect of different combinations of relative parameters in $J_{\rm RPC}$ by showing downstream performance on visual object recognition. We train SimCLRv2 on CIFAR-10 under $\alpha \in \{0, 0.1, 0.5, 1.0\}$, $\beta \in \{0.001, 0.005, 0.01\}$, and $\gamma \in \{0.1, 0.5, 1.0\}$. Empirically we observe that the best performance is $92.38\%$, which is achieved when $\alpha=0.5$, $\beta=0.01$, and $\gamma=0.1$. We could see the performances are bad when $\alpha$ is close to 0, $\beta = 0.001$ and $\gamma = 0.1$. We also notice that $\gamma = 1.0$ gives better results than $\gamma = 0.5$ and $\gamma = 0.1$ in most cases, despite the fact that the best performance is achieved when $\gamma=0.1$. Overall, the model can achieve strong performances when $\alpha = 0.5$, $\beta = 0.005$ or $0.01$ and $\gamma = 1.0$.  Our conclusion is that a small $\beta (0.005 \text{ to } 0.01)$ and large $\alpha (0.05)$ and $\gamma (1.0)$ can help the model extract useful information for downstream tasks, which is consistent with the regularization role of $\beta$.}

\begin{figure}[t!]
\centering
\includegraphics[width=\textwidth]{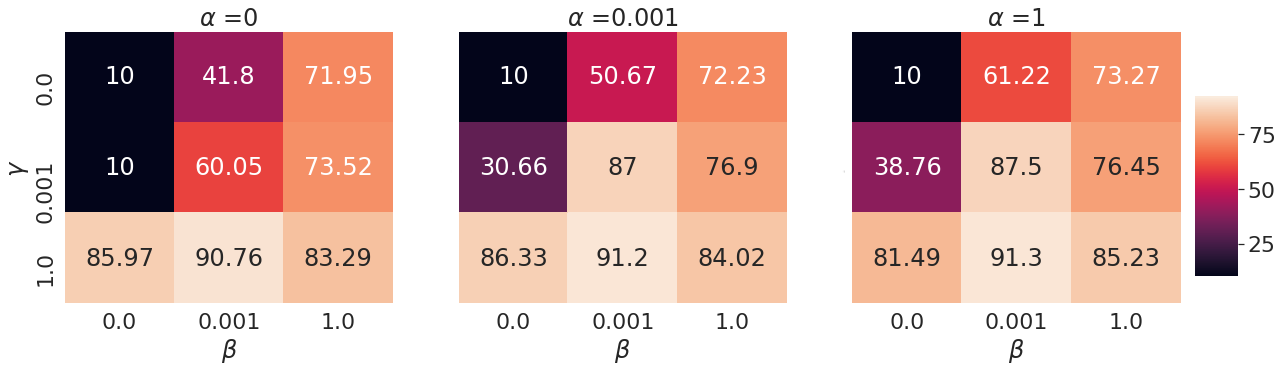}
\caption{\small Heatmaps of downstream task performance on CIFAR-10, using different $\alpha$, $\beta$ and $\gamma$ in the $J_{\rm RPC}$. We conclude that a nonzero $\alpha$, a small $\beta$ $(\beta = 0.001)$ and a large $\gamma (\gamma = 1.0)$ are crucial for better performance.}
\label{fig:relative_parameters_ablation}
\end{figure}

\subsection{Relation to Mutual Information Estimation}
\label{subsec:MI}

The presented approach also closely relates to mutual information estimation. For random variables $X$ and $Y$ with joint distribution $P_{XY}$ and product of marginals $P_XP_Y$, the mutual information is defined as $I(X;Y) = \KL(P_{XY} \Vert P_XP_Y)$. Lemma~\ref{lemma:ratio} states that given optimal solution $f^*(x,y)$ of $J_{\rm RPC}$, we can get the density ratio $r(x,y):=p(x,y)/p(x)p(y)$ as $r(x,y) = \frac{\gamma / \beta + \alpha}{1-\beta f^*(x,y)}-\frac{\gamma}{\beta}$. We can empirically estimate $\hat{r}(x, y)$ from the estimated $\hat{f}(x,y)$ via this transformation, and use $\hat{r}(x, y)$ to estimate mutual information \citep{tsai2020neural}. Specifically, $I(X;Y)\approx \frac{1}{n}\sum_{i=1}^n {\rm log}\,\hat{r}(x_i,y_i)$ with $(x_i,y_i)\sim P^{\otimes n}_{X,Y}$, where $P^{\otimes n}_{X,Y}$ is the uniformly sampled empirical distribution of $P_{X,Y}$.

% with probability density $p(x,y)$ and their product of marginal distributions $P_XP_Y$ with product of densities $p(x)p(y)$, 

% To see the link, we first show the optimal solution of our learning objective in \Eqref{eq:RPC_math_def} by replace $P$ with $P(X, Y)$ and $Q$ with $P(X)P(Y)$:
% \vspace{-2mm}

% \vspace{-5mm}
%\begin{proof}
%Take the functional derivative over $f(x,y)$ in~\eqref{eq:RPC_math_def}.
%\end{proof}

We follow prior work~\citep{poole2019variational,song2019understanding,tsai2020neural}
for the experiments. We consider $X$ and $Y$ as two $20$-dimensional Gaussians with correlation $\rho$, and our goal is to estimate the mutual information $I(X;Y)$.  Then, we perform a cubic transformation on $y$ so that $y\mapsto y^3$. The first task is referred to as {\bf Gaussian} task and the second is referred to as {\bf Cubic} task, where both have the ground truth $I(X;Y)=-10{\rm log}\,(1-\rho^2)$. The models are trained on $20,000$ steps with $I(X;Y)$ starting at $2$ and increased by $2$ per $4,000$ steps. Our method is compared with baseline methods $J_{\rm CPC}$~\citep{oord2018representation}, $J_{\rm NWJ}$~\citep{nguyen2010estimating}, $J_{\rm JS}$~\citep{nowozin2016f}, SMILE~\citep{song2019understanding} and Difference of Entropies (DoE)~\citep{mcallester2020formal}. All approaches use the same network design, learning rate, optimizer and minibatch size for a fair comparison. First, we observe $J_{\rm CPC}$~\citep{oord2018representation} has the smallest variance, while it exhibits a large bias (the estimated mutual information from $J_{\rm CPC}$ has an upper bound ${\rm log}({\rm batch\,size})$). Second, $J_{\rm NWJ}$~\citep{nguyen2010estimating} and $J_{\rm JSD}$~\citep{poole2019variational} have large variances, especially in the Cubic task. \citet{song2019understanding} pointed out the limitations of $J_{\rm CPC}$, $J_{\rm NWJ}$, and $J_{\rm JSD}$, and developed the SMILE method, which clips the value of the estimated density function to reduce the variance of the estimators.  DoE~\citep{mcallester2020formal} is neither a lower bound nor a upper bound of mutual information, but can achieve accurate estimates
when underlying mutual information is large.  $J_{\rm RPC}$ exhibits comparable bias and lower variance compared to the SMILE method, and is more stable than the DoE method. We would like to highlight our method's low-variance property, where we neither clip the values of the estimated density ratio nor impose an upper bound of our estimated mutual information.

\begin{figure}[t!]
\begin{center}
%\framebox[4.0in]{$\;$}
\includegraphics[scale=0.41]{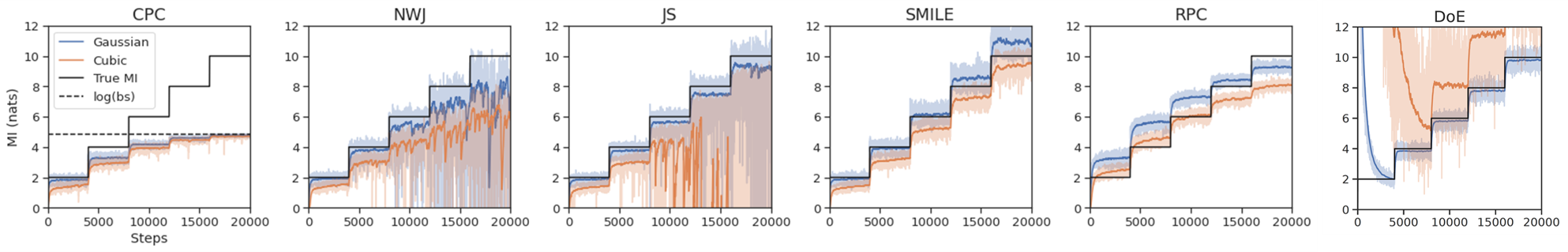}
\end{center}
\vspace{-5mm}
\caption{\small Mutual information estimation performed on 20-d correlated Gaussian distribution, with the correlation increasing each 4K steps. $J_{\rm RPC}$ exhibits smaller variance than SMILE and DoE, and smaller bias than $J_{\rm CPC}$.}
\label{fig:MI}
\end{figure}

\section{Related Work}

As a subset of unsupervised representation learning, self-supervised representation learning (SSL) adopts self-defined signals as supervision and uses the learned representation for downstream tasks, such as object detection and image captioning~\citep{liu2020self}. We categorize SSL work into two groups: when the signal is the input's hidden property or the corresponding view of the input. For the first group, for example, Jigsaw puzzle~\citep{noroozi2016unsupervised} shuffles the image patches and defines the SSL task for predicting the shuffled positions of the image patches. Other instances are Predicting Rotations~\citep{gidaris2018unsupervised} and Shuffle \& Learn~\citep{misra2016shuffle}. For the second group, the SSL task aims at modeling the co-occurrence of multiple views of data, via the contrastive or the predictive learning objectives~\citep{tsai2020demystifying}. The predictive objectives encourage reconstruction from one view of the data to the other, such as predicting the lower part of an image from its upper part (ImageGPT by~\citet{chen2020generative}). Comparing the contrastive with predictive learning approaches,~\citet{tsai2020demystifying} points out that the former requires less computational resources for a good performance but suffers more from the over-fitting problem.

Theoretical analysis~\citep{arora2019theoretical,tsai2020demystifying,tosh2020contrastive} suggests the contrastively learned representations can lead to a good downstream performance. Beyond the theory, \citet{tian2020makes} shows what matters more for the performance are 1) the choice of the contrastive learning objective; and 2) the creation of the positive and negative data pairs in the contrastive objective. Recent work~\citep{khosla2020supervised} extends the usage of contrastive learning from the self-supervised setting to the supervised setting. 
The supervised setting defines the positive pairs as the data from the same class in the contrastive objective, while the self-supervised setting defines the positive pairs as the data with different augmentations.

%As suggested by the theoretical analysis in \citet{arora2019theoretical}, contrastive learning methods aims to force representations of positive pairs to be close and representations of negative pairs to be distant. Previous methods like \citet{schroff2015facenet} tries to achieve this by applying a distance margin to separate the positive and negative samples, while \citet{mikolov2013distributed} use a softmax function to maximize the probability of positive samples. As a comparison, our method tries to maximize bounds of discrepancy between the empirical distribution of the positive pairs and the empirical distribution of the negative pairs. 

Our work also closely rates to the {\em skewed divergence} measurement between distributions~\citep{lee-1999-measures,lee2001effectiveness,nielsen2010family,yamada2013relative}. Recall that the usage of the relative parameters plays a crucial role to regularize our objective for its boundness and low variance. This idea is similar to the {\em skewed divergence} measurement, that when calculating the divergence between distributions $P$ and $Q$, instead of considering ${\rm D}(P\,\|\,Q)$, these approaches consider ${\rm D}(P\,\|\,\alpha P + (1-\alpha )Q)$ with $D$ representing the divergence and $0 < \alpha < 1$. A natural example is that the Jensen-Shannon divergence is a symmetric skewed KL divergence: ${D}_{\rm JS}(P\,\|\,Q) = 0.5 {D}_{\rm KL}(P\,\|\,0.5P+0.5Q) + 0.5 {D}_{\rm KL}(Q\,\|\,0.5P+0.5Q)$. Compared to the non-skewed counterpart, the skewed divergence has shown to have a more robust estimation for its value~\citep{lee-1999-measures,lee2001effectiveness,yamada2013relative}. Different from these works that focus on estimating the values of distribution divergence, we focus on learning self-supervised representations.

\section{Conclusion}
In this work, we present RPC, the Relative Predictive Coding, that achieves a good balance among the three challenges when modeling a contrastive learning objective: training stability, sensitivity to minibatch size, and downstream task performance. We believe this work brings an appealing option for training self-supervised models and inspires future work to design objectives for balancing the aforementioned three challenges. In the future, we are interested in applying RPC in other application domains and developing more principled approaches for better representation learning.

\section*{Acknowledgement}
This work was supported in part by the NSF IIS1763562, NSF Awards \#1750439 \#1722822, National Institutes of Health, IARPA
D17PC00340, ONR Grant N000141812861, and Facebook PhD Fellowship.
We would also like to acknowledge NVIDIA's GPU support and Cloud TPU support from Google's TensorFlow Research Cloud (TFRC).

%is inspired by the relative Chi-square divergence. Empirically it outperforms all other objectives in various visual and speech downstream tasks by self-supervised training, exhibits lower variance than DV, NWJ and JSD, and is also less sensitive to minibatch size changes than CPC. 

%We are excited about the prospect of contrastive self-supervised learning in future work, especially when performances on downstream tasks using unsupervised learning have caught up with or even surpassed performance under supervised setting \citep{chen2020big}. 

% Unlike CPC which selects on positive sample out of many negative samples, RPC computes the expectation statistics over the distribution of positive samples and distribution of negative samples, thus does not require a large batch size to make the task non-trivial.

\bibliography{iclr2021_conference}
\bibliographystyle{iclr2021_conference}

\appendix
\section{Appendix}

\subsection{Proof of Lemma 1 in the Main Text}
\begin{lemm}[Optimal Solution for $J_{\rm RPC}$, restating Lemma 1 in the main text] 
Let $$J_{\rm RPC}(X,Y):=\underset{f\in\mathcal{F}}{\rm sup}\,\mathbb{E}_{P_{XY}}[f(x,y)]-\alpha \mathbb{E}_{P_XP_Y}[f(x,y)]-\frac{\beta}{2} \mathbb{E}_{P_{XY}}\left[f^{2}(x,y)\right]-\frac{\gamma}{2}\mathbb{E}_{P_XP_Y}\left[f^{2}(x,y)\right]$$
and $r(x,y) = \frac{p(x,y)}{p(x)p(y)}$ be the density ratio. $J_{\rm RPC}$ has the optimal solution
$$f^*(x,y) = \frac{r(x,y) - \alpha}{\beta \, r(x,y) + \gamma}:=r_{\alpha, \beta, \gamma}(x,y)\,\,{\rm with}\,\,-\frac{\alpha}{\gamma}\leq r_{\alpha, \beta, \gamma} \leq \frac{1}{\beta}.$$
%\,\,{\rm with}\,\, -\frac{\alpha}{\gamma} \leq f^*(x,y) \leq \frac{1}{\beta}.
\label{lemma:ratio2}
\end{lemm}
\begin{proof}
The second-order functional derivative of the objective is
\begin{equation*}
-\beta dP_{X,Y} - \gamma dP_XP_Y,
\end{equation*}
which is always negative. The negative second-order functional derivative implies the objective has a supreme value. 
Then, take the first-order functional derivative $\frac{\partial J_{\rm RPC}}{\partial m}$ and set it to zero:
\begin{equation*}
dP_{X,Y} - \alpha \cdot dP_XP_Y -\beta \cdot f(x,y)\cdot dP_{X,Y} - \gamma \cdot f(x,y) \cdot dP_XP_Y=0.
\end{equation*}
We then get 
$$f^*(x,y) = \frac{dP_{X,Y} - \alpha \cdot dP_XP_Y}{\beta \cdot dP_{X,Y} + \gamma \cdot dP_XP_Y} = \frac{p(x,y)-\alpha p(x)p(y)}{\beta p(x,y) + \gamma p(x)p(y)} = \frac{r(x,y)-\alpha}{\beta r(x,y) + \gamma }.
$$ 
Since $0 \leq r(x,y) \leq \infty$, we have $-\frac{\alpha}{\gamma} \leq \frac{r(x,y)-\alpha}{\beta r(x,y) + \gamma } \leq \frac{1}{\beta}$. Hence, 
$$
\forall \beta \neq 0, \gamma \neq 0, f^*(x,y) := r_{\alpha, \beta, \gamma}(x,y) \,\,{\rm with}\,\, -\frac{\alpha}{\gamma} \leq r_{\alpha, \beta, \gamma} \leq \frac{1}{\beta} .
$$

%When $\Theta$ is large enough, by universal approximation theorem of neural networks~\cite{hornik1989multilayer}, the approximation in Proposition~\ref{prop:js} is tight, which means $\hat{f}^*_\theta(x,y) = m^*(x,y) = {\rm log}\,\frac{p(x,y)}{p(x)p(y)}$. 
\end{proof}

\subsection{Relation between $J_{\rm RPC}$ and $D_{\chi^2}$}
In this subsection, we aim to show the following: 1) $D_{\chi^2}(P_{XY}\,\|\,P_XP_Y) = \mathbb{E}_{P_XP_Y}[r^2(x,y)] - 1$; and 2) $J_{\rm RPC}(X,Y) = \frac{\beta + \gamma}{2} \mathbb{E}_{P'}[r_{\alpha, \beta, \gamma}^2(x,y)]$ by having $P' = \frac{\beta}{\beta+\gamma}P_{XY} + \frac{\gamma}{\beta+\gamma}P_XP_Y$ as the mixture distribution of $P_{XY}$ and $P_XP_Y$.

\begin{lemm}
$D_{\chi^2}(P_{XY}\,\|\,P_XP_Y) = \mathbb{E}_{P_XP_Y}[r^2(x,y)] - 1$
\label{lemm:chi}
\end{lemm}
\begin{proof}
By definition~\citep{nielsen2013chi}, 
\begin{equation*}
\begin{split}
D_{\chi^2}(P_{XY}\,\|\,P_XP_Y) & = \int \frac{\Big(dP_{XY}\Big)^2}{dP_XP_Y} -1 = \int \Big(\frac{dP_{XY}}{dP_XP_Y}\Big)^2 dP_XP_Y  -1 \\
& =  \int \Big(\frac{p(x,y)}{p(x)p(y)}\Big)^2 dP_XP_Y  -1 = \int r^2(x,y) dP_XP_Y  -1 \\
& = \mathbb{E}_{P_XP_Y}[r^2(x,y)] - 1.    
\end{split}
\end{equation*}
\end{proof}
\begin{lemm} Defining $P' = \frac{\beta}{\beta+\gamma}P_{XY} + \frac{\gamma}{\beta+\gamma}P_XP_Y$ as a mixture distribution of $P_{XY}$ and $P_XP_Y$, $J_{\rm RPC}(X,Y) = \frac{\beta + \gamma}{2} \mathbb{E}_{P'}[r_{\alpha, \beta, \gamma}^2(x,y)]$.
\label{lemma:rpc}
\end{lemm}
\begin{proof}
Plug in the optimal solution $f^*(x,y) = \frac{dP_{X,Y} - \alpha \cdot dP_XP_Y}{\beta \cdot dP_{X,Y} + \gamma \cdot dP_XP_Y}$ (see Lemma~\ref{lemma:ratio2}) into $J_{\rm RPC}$:
\begin{equation*}
\begin{split}
    J_{\rm RPC} & = \mathbb{E}_{P_{XY}}[f^*(x,y)]-\alpha \mathbb{E}_{P_XP_Y}[f^*(x,y)]-\frac{\beta}{2} \mathbb{E}_{P_{XY}}\left[{f^*}^2(x,y)\right]-\frac{\gamma}{2}\mathbb{E}_{P_XP_Y}\left[{f^*}^2(x,y)\right] \\
    & = \int f^*(x,y) \cdot \Big(dP_{XY} - \alpha \cdot dP_XP_Y\Big) - \frac{1}{2}{f^*}^2(x,y) \cdot \Big(\beta \cdot dP_{XY} + \gamma \cdot dP_XP_Y \Big) \\
    & = \int \frac{dP_{X,Y} - \alpha \cdot dP_XP_Y}{\beta \cdot dP_{X,Y} + \gamma \cdot dP_XP_Y} \Big(dP_{XY} - \alpha \cdot dP_XP_Y\Big) - \frac{1}{2}\Big( \frac{dP_{X,Y} - \alpha \cdot dP_XP_Y}{\beta \cdot dP_{X,Y} + \gamma \cdot dP_XP_Y} \Big)^2 \Big(\beta \cdot dP_{XY} + \gamma \cdot dP_XP_Y \Big) \\
    & = \frac{1}{2} \int \Big( \frac{dP_{X,Y} - \alpha \cdot dP_XP_Y}{\beta \cdot dP_{X,Y} + \gamma \cdot dP_XP_Y} \Big)^2 \Big(\beta \cdot dP_{XY} + \gamma \cdot dP_XP_Y \Big) \\
    & = \frac{\beta + \gamma}{2} \int \Big( \frac{dP_{X,Y} - \alpha \cdot dP_XP_Y}{\beta \cdot dP_{X,Y} + \gamma \cdot dP_XP_Y} \Big)^2 \Big(\frac{\beta}{\beta+\gamma} \cdot dP_{XY} + \frac{\gamma}{\beta+\gamma} \cdot dP_XP_Y \Big).
\end{split}
\end{equation*}
Since we define $r_{\alpha, \beta, \gamma} = \frac{dP_{X,Y} - \alpha \cdot dP_XP_Y}{\beta \cdot dP_{X,Y} + \gamma \cdot dP_XP_Y}$ and $P' = \frac{\beta}{\beta+\gamma}P_{XY} + \frac{\gamma}{\beta+\gamma}P_XP_Y$, 
$$
J_{\rm RPC} = \frac{\beta + \gamma}{2}\mathbb{E}_{P'}[r^2_{\alpha, \beta, \gamma}(x,y)].
$$
\end{proof}

\subsection{Proof of Proposition 1 in the Main Text}
The proof contains two parts: showing $0 \leq J_{\rm RPC} \leq \frac{1}{2\beta}+\frac{\alpha^2}{2\gamma}$ (see Section~\ref{subsec:bound}) and $\hat{J}^{m,n}_{\rm RPC}$ is a consistent estimator for $J_{\rm RPC}$ (see Section~\ref{subsec:consistency}).

\subsubsection{Boundness of $J_{\rm RPC}$}
\label{subsec:bound} 
\begin{lemm}[Boundness of $J_{\rm RPC}$]
$0 \leq J_{\rm RPC} \leq \frac{1}{2\beta}+\frac{\alpha^2}{2\gamma}$
\end{lemm}
\begin{proof}
Lemma~\ref{lemma:rpc} suggests $J_{\rm RPC}(X,Y) = \frac{\beta + \gamma}{2} \mathbb{E}_{P'}[r_{\alpha, \beta, \gamma}^2(x,y)]$ with $P' = \frac{\beta}{\beta+\gamma}P_{XY} + \frac{\gamma}{\beta+\gamma}P_XP_Y$ as the mixture distribution of $P_{XY}$ and $P_XP_Y$. Hence, it is obvious $J_{\rm RPC}(X,Y)\geq 0$.

We leverage the intermediate results in the proof of Lemma~\ref{lemma:rpc}:
\begin{equation*}
    \begin{split}
        J_{\rm RPC}(X,Y) & = \frac{1}{2} \int \Big( \frac{dP_{X,Y} - \alpha \cdot dP_XP_Y}{\beta \cdot dP_{X,Y} + \gamma \cdot dP_XP_Y} \Big)^2 \Big(\beta \cdot dP_{XY} + \gamma \cdot dP_XP_Y \Big)\\
        & = \frac{1}{2}\int dP_{X,Y} \Big( \frac{dP_{X,Y} - \alpha \cdot dP_XP_Y}{\beta \cdot dP_{X,Y} + \gamma \cdot dP_XP_Y} \Big) - \frac{\alpha}{2} \int dP_XP_Y \Big( \frac{dP_{X,Y} - \alpha \cdot dP_XP_Y}{\beta \cdot dP_{X,Y} + \gamma \cdot dP_XP_Y} \Big) \\
        & = \frac{1}{2}\mathbb{E}_{P_{XY}}[r_{\alpha, \beta, \gamma}(x,y)] - \frac{\alpha}{2}\mathbb{E}_{P_{X}P_{Y}}[r_{\alpha, \beta, \gamma}(x,y)].
    \end{split}
    \end{equation*}
Since $-\frac{\alpha}{\gamma} \leq r_{\alpha, \beta, \gamma} \leq \frac{1}{\beta}$, $J_{\rm RPC}(X,Y)\leq \frac{1}{2\beta} + \frac{\alpha^2}{2\gamma}$. 
\end{proof}

\subsubsection{Consistency}
\label{subsec:consistency}

We first recall the definition of the estimation of $J_{\rm RPC}$:
\begin{defn}[$\hat{J}^{m,n}_{\rm RPC}$, empirical estimation of $J_{\rm RPC}$, restating Definition 1 in the main text] We parametrize $f$ via a family of neural networks $\mathcal{F}_\Theta:= \{f_\theta: \theta\in\Theta\subseteq\mathbb{R}^d
%, -\frac{\alpha}{\gamma}\leq f_\theta \leq \frac{1}{\beta}
\}$ where $d \in\mathbb{N}$ and $\Theta$ is compact. Let $\{x_i, y_i\}_{i=1}^n$ be $n$ samples drawn uniformly at random from $P_{XY}$ and $\{x'_j, y'_j\}_{j=1}^m$ be $m$ samples drawn uniformly at random from $P_XP_Y$. Then,
$$\hat{J}^{m,n}_{\rm RPC} = \sup_{{f}_\theta \in \mathcal{F}_\Theta} \frac{1}{n}\sum_{i=1}^n f_\theta(x_i,y_i) -  \frac{1}{m}\sum_{j=1}^m \alpha f_\theta(x'_j,y'_j) - \frac{1}{n}\sum_{i=1}^n  \frac{\beta}{2} f_\theta^2(x_i,y_i) -  \frac{1}{m}\sum_{j=1}^m \frac{\gamma}{2} f_\theta^2(x'_j,y'_j).
$$
\label{def:RPC2}
\end{defn}

Our goal is to show that $\hat{J}^{m,n}_{\rm RPC}$ is a consistent estimator for $J_{\rm RPC}$. We begin with the following definition:
\begin{equation}
\hat{J}^{m,n}_{{\rm RPC}, \theta} := \frac{1}{n}\sum_{i=1}^n f_\theta(x_i,y_i) -  \frac{1}{m}\sum_{j=1}^m \alpha f_\theta(x'_j,y'_j) - \frac{1}{n}\sum_{i=1}^n  \frac{\beta}{2} f_\theta^2(x_i,y_i) -  \frac{1}{m}\sum_{j=1}^m \frac{\gamma}{2} f_\theta^2(x'_j,y'_j)
\label{eq:RPC_theta}
\end{equation}
and 
\begin{equation}
\mathbb{E}\Big[\hat{J}_{{\rm RPC}, \theta}\Big] := \mathbb{E}_{P_{XY}}[f_\theta(x,y)] - \alpha \mathbb{E}_{P_XP_Y}[f_\theta(x,y)] - \frac{\beta}{2} \mathbb{E}_{P_{XY}}[f^2_\theta(x,y)] - \frac{\gamma}{2} \mathbb{E}_{P_XP_Y}[f^2_\theta(x,y)].
\label{eq:RPC_expected}
\end{equation}
Then, we follow the steps:
\begin{itemize}
    \item The first part is about estimation. We show that, with high probability, $\hat{J}^{m,n}_{{\rm RPC}, \theta}$ is close to $\mathbb{E}\Big[\hat{J}_{{\rm RPC}, \theta}\Big]$, for any given $\theta$.
    \item The second part is about approximation. We will apply the universal approximation lemma of neural networks~\citep{hornik1989multilayer} to show that there exists a network $\theta^*$ such that $\mathbb{E}\Big[\hat{J}_{{\rm RPC}, \theta^*}\Big]$ is close to $J_{\rm RPC}$.
\end{itemize}

\paragraph{Part I - Estimation: With high probability, $\hat{J}^{m,n}_{{\rm RPC}, \theta}$ is close to $\mathbb{E}\Big[\hat{J}_{{\rm RPC}, \theta}\Big]$, for any given $\theta$.}

Throughout the analysis on the uniform convergence, we need the assumptions on the boundness and smoothness of the function $f_\theta$. Since we show the optimal function $f$ is bounded in $J_{\rm RPC}$, we can use the same bounded values for $f_\theta$ without losing too much precision. The smoothness of the function suggests that the output of the network should only change slightly when only slightly perturbing the parameters. Specifically, the two assumptions are as follows:
\begin{assump}[boundness of $f_\theta$]
There exist universal constants such that $\forall f_\theta \in \mathcal{F}_\Theta$, $C_L \leq f_\theta \leq C_U$. For notations simplicity, we let $M = C_U -  C_L$ be the range of $f_\theta$ and $U = {\rm max}\,\{|C_U|, |C_L|\}$ be the maximal absolute value of $f_\theta$. In the paper, we can choose to constrain that $C_L = -\frac{\alpha}{\gamma}$ and $C_U = \frac{1}{\beta}$ since the optimal function $f^*$ has $-\frac{\alpha}{\gamma}\leq f^* \leq \frac{1}{\beta}$.

\label{assump:bound_f}
\end{assump}
\begin{assump}[smoothness of $f_\theta$]
There exists constant $\rho > 0$ such that $\forall (x,y) \in (\mathcal{X}\times \mathcal{Y})$ and $\theta_1, \theta_2 \in \Theta$, $|f_{\theta_1}(x,y)-f_{\theta_2}(x,y)|\leq \rho |\theta_1 - \theta_2|$.
\label{assump:smooth_f}
\end{assump}

Now, we can bound the rate of uniform convergence of a function class in terms of covering number~\citep{bartlett1998sample}:
\begin{lemm}[Estimation] Let $\epsilon > 0$ and $\mathcal{N}(\Theta, \epsilon)$ be the covering number of $\Theta$ with radius $\epsilon$. Then,
\begin{equation*}
\begin{split}
    &\Pr\left(\sup_{f_\theta\in\mathcal{F}_{\Theta}} \left|\hat{J}^{m,n}_{{\rm RPC}, \theta} - \mathbb{E}\Big[\hat{J}_{{\rm RPC}, \theta}\Big]\right| \geq \eps\right)
    \\
    \leq & 2\mathcal{N}(\Theta, \frac{\epsilon}{4\rho \big(1+\alpha + 2(\beta+\gamma)U\big)})\Bigg({\rm exp}\Big(-\frac{n \epsilon^2}{32M^2}\Big) + {\rm exp}\Big(-\frac{m \epsilon^2}{32M^2\alpha^2}\Big) + {\rm exp}\Big(-\frac{n \epsilon^2}{32U^2\beta^2}\Big) + {\rm exp}\Big(-\frac{m \epsilon^2}{32U^2\gamma^2}\Big)\Bigg).
\end{split}
\end{equation*}
\label{lemm:estimation}
\end{lemm}
\begin{proof}
For notation simplicity, we define the operators
\begin{itemize}
    \item $P(f) = \mathbb{E}_{P_{XY}}[f(x,y)]$ and $P_n(f) = \frac{1}{n}\sum_{i=1}^n f(x_i, y_i)$
    \item $Q(f) = \mathbb{E}_{P_XP_Y}[f(x,y)]$ and $Q_m(f) = \frac{1}{m}\sum_{j=1}^m f(x'_j, y'_j)$
\end{itemize}

Hence, 
\begin{equation*}
\begin{split}
    &\left|\hat{J}^{m,n}_{{\rm RPC}, \theta} - \mathbb{E}\Big[\hat{J}_{{\rm RPC}, \theta}\Big]\right| \\
    = & \left| P_n(f_\theta) - P(f_\theta) - \alpha Q_m(f_\theta) + \alpha Q(f_\theta) - \beta P_n(f_\theta^2) + \beta P(f_\theta^2) - \gamma Q_m(f_\theta^2) + \gamma Q(f_\theta^2) \right| \\
    \leq & \left| P_n(f_\theta) - P(f_\theta) \right| + \alpha \left| Q_m(f_\theta) - Q(f_\theta)\right| + \beta \left| P_n(f_\theta^2) - P(f_\theta^2) \right| + \gamma \left| Q_m(f_\theta^2) - Q(f_\theta^2) \right| 
\end{split}
\end{equation*}

Let $\epsilon' = \frac{\epsilon}{4\rho \big(1+\alpha + 2(\beta+\gamma)U\big)}$ and $T := \mathcal{N}(\Theta,\epsilon')$. Let $C=\{f_{\theta_1}, f_{\theta_2}, \cdots, f_{\theta_T}\}$ with $\{\theta_1, \theta_2, \cdots, \theta_T\}$ be such that $B_\infty(\theta_1, \epsilon')$, 
$\cdots$, $B_\infty(\theta_T, \epsilon')$ are $\epsilon'$ cover. Hence, for any $f_\theta \in \mathcal{F}_\Theta$, there is an $f_{\theta_k} \in C$ such that $\|\theta - \theta_k\|_\infty \leq \epsilon'$.

Then, for any $f_{\theta_k} \in C$: 
\begin{equation*}
\begin{split}
    &\left|\hat{J}^{m,n}_{{\rm RPC}, \theta} - \mathbb{E}\Big[\hat{J}_{{\rm RPC}, \theta}\Big]\right| \\
    \leq & \left| P_n(f_\theta) - P(f_\theta) \right| + \alpha \left| Q_m(f_\theta) - Q(f_\theta)\right| + \beta \left| P_n(f_\theta^2) - P(f_\theta^2) \right| + \gamma \left| Q_m(f_\theta^2) - Q(f_\theta^2) \right| \\
    \leq & \left| P_n(f_{\theta_k}) - P(f_{\theta_k}) \right| + \left| P_n(f_{\theta}) - P_n(f_{\theta_k}) \right| + \left| P(f_{\theta}) - P(f_{\theta_k}) \right| \\
    & + \alpha\bigg( \left| Q_m(f_{\theta_k}) - Q(f_{\theta_k}) \right| + \left| Q_m(f_{\theta}) - Q_m(f_{\theta_k}) \right| + \left| Q(f_{\theta}) - Q(f_{\theta_k}) \right| \bigg) \\
    & + \beta\bigg( \left| P_n(f^2_{\theta_k}) - P(f^2_{\theta_k}) \right| + \left| P_n(f^2_{\theta}) - P_n(f^2_{\theta_k}) \right| + \left| P(f^2_{\theta}) - P(f^2_{\theta_k}) \right| \bigg) \\
    & + \gamma\bigg( \left| Q_m(f^2_{\theta_k}) - Q(f^2_{\theta_k}) \right| + \left| Q_m(f^2_{\theta}) - Q_m(f^2_{\theta_k}) \right| + \left| Q(f^2_{\theta}) - Q(f^2_{\theta_k}) \right| \bigg) \\
    \leq & \left| P_n(f_{\theta_k}) - P(f_{\theta_k}) \right| + \rho \|\theta - \theta_k\| + \rho \|\theta - \theta_k\| \\
    & + \alpha\bigg( \left| Q_m(f_{\theta_k}) - Q(f_{\theta_k}) \right| + \rho \|\theta - \theta_k\| + \rho \|\theta - \theta_k\| \bigg) \\
    & + \beta\bigg( \left| P_n(f^2_{\theta_k}) - P(f^2_{\theta_k}) \right| + 2 \rho U \|\theta - \theta_k\| + 2 \rho U \|\theta - \theta_k\| \bigg) \\
    & + \gamma\bigg( \left| Q_m(f^2_{\theta_k}) - Q(f^2_{\theta_k}) \right| + 2 \rho U \|\theta - \theta_k\| + 2 \rho U \|\theta - \theta_k\| \bigg) \\
    = &  \left| P_n(f_{\theta_k}) - P(f_{\theta_k}) \right| + \alpha \left| Q_m(f_{\theta_k}) - Q(f_{\theta_k}) \right| + \beta \left| P_n(f^2_{\theta_k}) - P(f^2_{\theta_k}) \right| + \gamma \left| Q_m(f^2_{\theta_k}) - Q(f^2_{\theta_k}) \right| \\
    & + 2 \rho \big(1 + \alpha + 2(\beta+\gamma)U\big) \|\theta - \theta_k\| \\
    \leq & \left| P_n(f_{\theta_k}) - P(f_{\theta_k}) \right| + \alpha \left| Q_m(f_{\theta_k}) - Q(f_{\theta_k}) \right| + \beta \left| P_n(f^2_{\theta_k}) - P(f^2_{\theta_k}) \right| + \gamma \left| Q_m(f^2_{\theta_k}) - Q(f^2_{\theta_k}) \right| + \frac{\epsilon}{2},
\end{split}
\end{equation*}
where
\begin{itemize}
    \item $\left| P_n(f_\theta) - P_n(f_{\theta_k}) \right| \leq \rho \|\theta - \theta_k \|$ due to Assumption~\ref{assump:smooth_f}, and the result also applies for $\left| P(f_\theta) - P(f_{\theta_k}) \right|$, $\left| Q_m(f_\theta) - Q_m(f_{\theta_k}) \right| $, and $\left| Q(f_\theta) - Q(f_{\theta_k}) \right|$.
    \item $\left| P_n(f^2_\theta) - P_n(f^2_{\theta_k}) \right| \leq 2 \|f_\theta \|_\infty \rho \|\theta - \theta_k\| \leq 2 \rho U \|\theta - \theta_k\|$ due to Assumptions~\ref{assump:bound_f} and~\ref{assump:smooth_f}. The result also applies for $\left| P(f^2_\theta) - P(f^2_{\theta_k}) \right|$, $\left| Q_m(f^2_\theta) - Q_m(f^2_{\theta_k}) \right|$, and $\left| Q(f^2_\theta) - Q(f^2_{\theta_k}) \right|$.
\end{itemize}

Hence, 
\begin{equation*}
\begin{split}
    & \Pr\left(\sup_{f_\theta\in\mathcal{F}_{\Theta}} \left|\hat{J}^{m,n}_{{\rm RPC}, \theta} - \mathbb{E}\Big[\hat{J}_{{\rm RPC}, \theta}\Big]\right| \geq \eps\right) \\
    \leq & \Pr\left(\underset{f_{\theta_k}\in C}{\rm max} \left| P_n(f_{\theta_k}) - P(f_{\theta_k}) \right| + \alpha \left| Q_m(f_{\theta_k}) - Q(f_{\theta_k}) \right| + \beta \left| P_n(f^2_{\theta_k}) - P(f^2_{\theta_k}) \right| + \gamma \left| Q_m(f^2_{\theta_k}) - Q(f^2_{\theta_k}) \right| + \frac{\epsilon}{2}\geq \eps \right) \\
    = & \Pr\left(\underset{f_{\theta_k}\in C}{\rm max} \left| P_n(f_{\theta_k}) - P(f_{\theta_k}) \right| + \alpha \left| Q_m(f_{\theta_k}) - Q(f_{\theta_k}) \right| + \beta \left| P_n(f^2_{\theta_k}) - P(f^2_{\theta_k}) \right| + \gamma \left| Q_m(f^2_{\theta_k}) - Q(f^2_{\theta_k}) \right| \geq \frac{\epsilon}{2} \right) \\
    \leq & \sum_{k=1}^T \Pr\left( \left| P_n(f_{\theta_k}) - P(f_{\theta_k}) \right| + \alpha \left| Q_m(f_{\theta_k}) - Q(f_{\theta_k}) \right| + \beta \left| P_n(f^2_{\theta_k}) - P(f^2_{\theta_k}) \right| + \gamma \left| Q_m(f^2_{\theta_k}) - Q(f^2_{\theta_k}) \right| \geq \frac{\epsilon}{2} \right) \\
    \leq & \sum_{k=1}^T \Pr\left( \left| P_n(f_{\theta_k}) - P(f_{\theta_k}) \right| \geq \frac{\epsilon}{8} \right) + \Pr\left(\alpha \left| Q_m(f_{\theta_k}) - Q(f_{\theta_k}) \right| \geq \frac{\epsilon}{8} \right) \\
    & \,\,\,\,\,\,\,\,\,\, + \Pr\left(\beta \left| P_n(f^2_{\theta_k}) - P(f^2_{\theta_k}) \right| \geq \frac{\epsilon}{8} \right) + \Pr\left( \gamma \left| Q_m(f^2_{\theta_k}) - Q(f^2_{\theta_k}) \right| \geq \frac{\epsilon}{8} \right).
\end{split} 
\end{equation*}
With Hoeffding's inequality, 
\begin{itemize}
    \item $\Pr\left( \left| P_n(f_{\theta_k}) - P(f_{\theta_k}) \right| \geq \frac{\epsilon}{8} \right) \leq 2 {\rm exp}\Big(-\frac{n \epsilon^2}{32M^2}\Big)$ 
    \item $\Pr\left(\alpha \left| Q_m(f_{\theta_k}) - Q(f_{\theta_k}) \right| \geq \frac{\epsilon}{8} \right) \leq 2 {\rm exp}\Big(-\frac{m \epsilon^2}{32M^2\alpha^2}\Big)$
    \item $\Pr\left(\beta \left| P_n(f^2_{\theta_k}) - P(f^2_{\theta_k}) \right| \geq \frac{\epsilon}{8} \right) \leq 2 {\rm exp}\Big(-\frac{n \epsilon^2}{32U^2\beta^2}\Big)$
    \item $\Pr\left( \gamma \left| Q_m(f^2_{\theta_k}) - Q(f^2_{\theta_k}) \right| \geq \frac{\epsilon}{8} \right) \leq 2 {\rm exp}\Big(-\frac{m \epsilon^2}{32U^2\gamma^2}\Big)$
\end{itemize}
To conclude, 
\begin{equation*}
\begin{split}
    & \Pr\left(\sup_{f_\theta\in\mathcal{F}_{\Theta}} \left|\hat{J}^{m,n}_{{\rm RPC}, \theta} - \mathbb{E}\Big[\hat{J}_{{\rm RPC}, \theta}\Big]\right| \geq \eps\right) \\
    \leq & 2\mathcal{N}(\Theta, \frac{\epsilon}{4\rho \big(1+\alpha + 2(\beta+\gamma)U\big)})\Bigg({\rm exp}\Big(-\frac{n \epsilon^2}{32M^2}\Big) + {\rm exp}\Big(-\frac{m \epsilon^2}{32M^2\alpha^2}\Big) + {\rm exp}\Big(-\frac{n \epsilon^2}{32U^2\beta^2}\Big) + {\rm exp}\Big(-\frac{m \epsilon^2}{32U^2\gamma^2}\Big)\Bigg).
\end{split}
\end{equation*}

\end{proof}

\paragraph{Part II - Approximation: Neural Network Universal Approximation.}

We leverage the universal function approximation lemma of neural network
\begin{lemm}[Approximation \citep{hornik1989multilayer}]
Let $\eps > 0$. There exists $d \in\mathbb{N}$ and a family of neural networks $\mathcal{F}_\Theta := \{f_\theta: \theta\in\Theta\subseteq\mathbb{R}^d\}$ where $\Theta$ is compact, such that $
\underset{f_\theta\in\mathcal{F}_{\Theta}}{\rm inf}\left|\mathbb{E}\Big[\hat{J}_{{\rm RPC}, \theta}\Big] - J_{\rm RPC} \right|\leq \eps$.
\label{lemma:nn}
\end{lemm}

\paragraph{Part III - Bringing everything together.}

Now, we are ready to bring the estimation and approximation together to show that there exists a neural network $\theta^*$ such that, with high probability, $\hat{J}^{m,n}_{{\rm RPC}, \theta}$ can approximate $J_{\rm RPC}$ with $n'={\rm min}\,\{n,m\}$ at a rate of $O(1/\sqrt{n'})$:
\begin{prop} With probability at least $1-\delta$, $\exists \theta^*\in\Theta$, $|J_{\rm RPC} - \hat{J}^{m,n}_{\rm RPC, \theta}| =  O(\sqrt{\frac{d + {\rm log}\,(1/\delta)}{n'}}),$ where $n' = {\rm min}\,\{n, m\}$. 
\label{prop:RPC_bound2}
\end{prop}
\begin{proof}
The proof follows by combining Lemma~\ref{lemm:estimation} and~\ref{lemma:nn}. 

First, Lemma~\ref{lemma:nn} suggests, $\exists\theta^*\in\Theta$, 
$$
\left|\mathbb{E}\Big[\hat{J}_{{\rm RPC}, \theta^*}\Big] - J_{\rm RPC} \right|\leq \frac{\eps}{2}.
$$

Next, we perform analysis on the estimation error, aiming to find $n,m$ and the corresponding probability, such that 
$$
\left|\hat{J}^{m,n}_{\rm RPC, \theta} - \mathbb{E}\Big[\hat{J}_{{\rm RPC}, \theta^*}\Big] \right|\leq \frac{\eps}{2}.
$$
Applying Lemma~\ref{lemm:estimation} with the covering number of the neural network: \bigg($\mathcal{N}(\Theta, \epsilon) = O\Big({\rm exp}\big(d \,{\rm log}\,(1/\epsilon)\big) \Big)$~\citep{anthony2009neural}\bigg) and let $n' = {\rm min}\{n,m\}$:
\begin{equation*}
\small
\begin{split}
    & \Pr\left(\sup_{f_\theta\in\mathcal{F}_{\Theta}} \left|\hat{J}^{m,n}_{{\rm RPC}, \theta} - \mathbb{E}\Big[\hat{J}_{{\rm RPC}, \theta}\Big]\right| \geq \frac{\eps}{2}\right) \\
    \leq & 2\mathcal{N}(\Theta, \frac{\epsilon}{8\rho \big(1+\alpha + 2(\beta+\gamma)U\big)})\Bigg({\rm exp}\Big(-\frac{n \epsilon^2}{128M^2}\Big) + {\rm exp}\Big(-\frac{m \epsilon^2}{128M^2\alpha^2}\Big) + {\rm exp}\Big(-\frac{n \epsilon^2}{128U^2\beta^2}\Big) + {\rm exp}\Big(-\frac{m \epsilon^2}{128U^2\gamma^2}\Big)\Bigg)\\\
    = & O\Big({\rm exp}\big(d\,{\rm log}\,(1/\epsilon) - n'\epsilon^2\big)\Big), 
\end{split}
\end{equation*}
where the big-O notation absorbs all the constants that do not require in the following derivation. Since we want to bound the probability with $1-\delta$, we solve the $\epsilon$ such that
$$
{\rm exp}\big(d\,{\rm log}\,(1/\epsilon) - n'\epsilon^2\big) \leq \delta.
$$
With ${\rm log}\,(x) \leq x -1$, 
$$n'\epsilon^2 + d (\epsilon-1) \geq n'\epsilon^2 + d {\rm log}\,\epsilon \geq {\rm log}\,(1/\delta),
$$
where this inequality holds when
$$
\epsilon = O\bigg(\sqrt{\frac{d+{\rm log}\,(1/\delta)}{n'}}\bigg).
$$
\end{proof}

\subsection{Proof of Proposition 2 in the Main Text - From an Asymptotic Viewpoint}
Here, we provide the variance analysis on $\hat{J}^{m,n}_{\rm RPC}$ via an asymptotic viewpoint. First, assuming the network is correctly specified, and hence there exists a network parameter $\theta^*$ satisfying $f^*(x,y) = f_{\theta^*}(x,y) = r_{\alpha, \beta, \gamma}(x,y)$. Then we recall that $\hat{J}^{m,n}_{\rm RPC}$ is a consistent estimator of ${J}^{\rm RPC}$ (see Proposition~\ref{prop:RPC_bound2}), and under regular conditions, the estimated network parameter $\hat{\theta}$ in $\hat{J}^{m,n}_{\rm RPC}$ satisfying the asymptotic normality in the large sample limit (see Theorem 5.23 in~\citep{van2000asymptotic}). 
%Hence, when $m,n \rightarrow \infty$, $\hat{J}^{m,n}_{\rm RPC, \hat{\theta}} \rightarrow J_{\rm RPC}$ and $\hat{\theta} \rightarrow \theta^*$. Then, 
We recall the definition of $\hat{J}^{m,n}_{{\rm RPC}, \theta}$ in~\eqref{eq:RPC_theta} and let $n' = {\rm min}\{n,m\}$, the asymptotic expansion of $\hat{J}^{m,n}_{\rm RPC}$ has
\begin{equation}
\begin{split}
\hat{J}^{m,n}_{{\rm RPC}, \theta^*} &= \hat{J}^{m,n}_{{\rm RPC}, \hat{\theta}} + \dot{\hat{J}}^{m,n}_{{\rm RPC}, \hat{\theta}}(\theta^* - \hat{\theta}) + o(\| \theta^*  - \hat{\theta}\|) \\
& = \hat{J}^{m,n}_{{\rm RPC}, \hat{\theta}} + \dot{\hat{J}}^{m,n}_{{\rm RPC}, \hat{\theta}}(\theta^* - \hat{\theta}) +  o_p(\frac{1}{\sqrt{n'}}) \\
& = \hat{J}^{m,n}_{{\rm RPC}, \hat{\theta}} +  o_p(\frac{1}{\sqrt{n'}}),
\end{split}
\label{eq:asym1}
\end{equation}
where $ \dot{\hat{J}}^{m,n}_{{\rm RPC}, \hat{\theta}} = 0$ since $\hat{\theta}$ is the estimation from $\hat{J}^{m,n}_{{\rm RPC}} = \underset{f_\theta \in \mathcal{F}_\Theta}{\rm sup}\,\hat{J}^{m,n}_{{\rm RPC}, \theta}$. 

Next, we recall the definition in~\eqref{eq:RPC_expected}:
$$\mathbb{E}[\hat{J}_{\rm RPC, \hat{\theta}}] = \mathbb{E}_{P_{XY}} [f_{\hat{\theta}}(x,y)] - \alpha \mathbb{E}_{P_XP_Y} [f_{\hat{\theta}}(x,y)] - \frac{\beta}{2}\mathbb{E}_{P_{XY}} [ f_{\hat{\theta}}^2(x,y)] - \frac{\gamma}{2} \mathbb{E}_{P_XP_Y} [f_{\hat{\theta}}^2(x,y)].$$
Likewise, the asymptotic expansion of $\mathbb{E}[\hat{J}_{\rm RPC, \theta}]$ has
\begin{equation}
\begin{split}
    \mathbb{E}[\hat{J}_{\rm RPC, \hat{\theta}}] &= \mathbb{E}[\hat{J}_{\rm RPC, \theta^*}] + \mathbb{E}[\dot{\hat{J}}_{\rm RPC, \theta^*}](\hat{\theta} - \theta^*) + o(\| \hat{\theta} - \theta^* \|) \\
    & = \mathbb{E}[\hat{J}_{\rm RPC, \theta^*}] + \mathbb{E}[\dot{\hat{J}}_{\rm RPC, \theta^*}](\hat{\theta} - \theta^*) + o_p(\frac{1}{\sqrt{n'}}) \\
    & = \mathbb{E}[\hat{J}_{\rm RPC, \theta^*}] + o_p(\frac{1}{\sqrt{n'}}),
\end{split}
\label{eq:asym2}
\end{equation}
where $\mathbb{E}[\dot{\hat{J}}_{\rm RPC, \theta^*}] = 0$ since $\mathbb{E}[\hat{J}_{\rm RPC, \theta^*}] = J_{\rm RPC}$ and $\theta^*$ satisfying $f^*(x,y)=f_{\theta^*}(x,y)$.

Combining equations~\ref{eq:asym1} and~\ref{eq:asym2}:
\begin{equation*}
\begin{split}
    \hat{J}^{m,n}_{{\rm RPC}, \hat{\theta}} - \mathbb{E}[\hat{J}_{\rm RPC, \hat{\theta}}]= & \hat{J}^{m,n}_{{\rm RPC}, \theta^*} - J_{\rm RPC} +  o_p(\frac{1}{\sqrt{n'}}) \\
    = & \frac{1}{n}\sum_{i=1}^n f_\theta^*(x_i,y_i) - \alpha \frac{1}{m}\sum_{j=1}^m f_\theta^*(x'_j,y'_j) - \frac{\beta}{2} \frac{1}{n}\sum_{i=1}^n f_{\theta^*}^2(x_i,y_i) - \frac{\gamma}{2} \frac{1}{m}\sum_{j=1}^m f_{\theta^*}^2(x'_j,y'_j) \\
    & - \mathbb{E}_{P_{XY}}[f^*(x,y)]+\alpha \mathbb{E}_{P_XP_Y}[f^*(x,y)]+\frac{\beta}{2} \mathbb{E}_{P_{XY}}\left[{f^*}^2(x,y)\right]+\frac{\gamma}{2}\mathbb{E}_{P_XP_Y}\left[{f^*}^2(x,y)\right] +  o_p(\frac{1}{\sqrt{n'}}) \\
    = & \frac{1}{n}\sum_{i=1}^n r_{\alpha, \beta, \gamma}(x_i,y_i) - \alpha \frac{1}{m}\sum_{j=1}^m r_{\alpha, \beta, \gamma}(x'_j,y'_j) - \frac{\beta}{2} \frac{1}{n}\sum_{i=1}^n r_{\alpha, \beta, \gamma}^2(x_i,y_i) - \frac{\gamma}{2} \frac{1}{m}\sum_{j=1}^m r_{\alpha, \beta, \gamma}^2(x'_j,y'_j)  \\
    & - \mathbb{E}_{P_{XY}}[r_{\alpha, \beta, \gamma}(x,y)]+\alpha \mathbb{E}_{P_XP_Y}[r_{\alpha, \beta, \gamma}(x,y)]+\frac{\beta}{2} \mathbb{E}_{P_{XY}}\left[r_{\alpha, \beta, \gamma}^2(x,y)\right]+\frac{\gamma}{2}\mathbb{E}_{P_XP_Y}\left[r_{\alpha, \beta, \gamma}^2(x,y)\right] \\
    & +  o_p(\frac{1}{\sqrt{n'}}) \\
    = & \frac{1}{\sqrt{n}}\cdot \frac{1}{\sqrt{n}} \sum_{i=1}^n\Bigg(  r_{\alpha, \beta, \gamma}(x_i,y_i) - \frac{\beta}{2}r_{\alpha, \beta, \gamma}^2(x_i,y_i) -\mathbb{E}_{P_{XY}}\bigg[r_{\alpha, \beta, \gamma}(x,y) - \frac{\beta}{2}r_{\alpha, \beta, \gamma}^2(x,y)\bigg] \Bigg) \\
    & - \frac{1}{\sqrt{m}} \cdot \frac{1}{\sqrt{m}} \sum_{j=1}^m \Bigg( \alpha r_{\alpha, \beta, \gamma}(x'_j,y'_j) + \frac{\gamma}{2}r^2_{\alpha, \beta, \gamma}(x'_j,y'_j) - \mathbb{E}_{P_XP_Y}\bigg[\alpha r_{\alpha, \beta, \gamma}(x,y) + \frac{\gamma}{2}r^2_{\alpha, \beta, \gamma}(x,y)\bigg] \Bigg)\\
    & + o_p(\frac{1}{\sqrt{n'}}).
\end{split}
\end{equation*}

Therefore, the asymptotic Variance of $\hat{J}^{m,n}_{{\rm RPC}}$ is
\begin{equation*}
\begin{split}
    \Var [\hat{J}^{m,n}_{{\rm RPC}}] = \frac{1}{n} \Var_{P_{XY}} [r_{\alpha, \beta, \gamma}(x,y) - \frac{\beta}{2} r^2_{\alpha, \beta, \gamma}(x,y)] + \frac{1}{m} \Var_{P_XP_Y} [\alpha r_{\alpha, \beta, \gamma}(x,y) +  \frac{\gamma}{2} r^2_{\alpha, \beta, \gamma}(x,y)] + o(\frac{1}{n'}). 
\end{split}
\end{equation*}

First, we look at $\Var_{P_{XY}} [r_{\alpha, \beta, \gamma}(x,y) - \frac{\beta}{2} r^2_{\alpha, \beta, \gamma}(x,y)]$. Since $\beta > 0$ and $-\frac{\alpha}{\gamma}\leq r_{\alpha, \beta, \gamma} \leq \frac{1}{\beta}$, simple calculation gives us $
-\frac{2\alpha \gamma + \beta \alpha^2}{2 \gamma^2}\leq r_{\alpha, \beta, \gamma}(x,y) - \frac{\beta}{2} r^2_{\alpha, \beta, \gamma}(x,y) \leq \frac{1}{2\beta}$. Hence, 
$$
\Var_{P_{XY}} [r_{\alpha, \beta, \gamma}(x,y) - \frac{\beta}{2} r^2_{\alpha, \beta, \gamma}(x,y)] \leq {\rm max}\bigg\{\Big(\frac{2\alpha \gamma + \beta \alpha^2}{2 \gamma^2}\Big)^2, \Big(\frac{1}{2\beta}\Big)^2\bigg\}.
$$
Next, we look at $\Var_{P_XP_Y} [\alpha r_{\alpha, \beta, \gamma}(x,y) +  \frac{\gamma}{2} r^2_{\alpha, \beta, \gamma}(x,y)]$. Since $\alpha \geq 0, \gamma > 0$ and $-\frac{\alpha}{\gamma}\leq r_{\alpha, \beta, \gamma} \leq \frac{1}{\beta}$,
simple calculation gives us $
-\frac{\alpha^2}{2 \gamma}\leq \alpha r_{\alpha, \beta, \gamma}(x,y) +  \frac{\gamma}{2} r^2_{\alpha, \beta, \gamma}(x,y) \leq \frac{2\alpha \beta + \gamma}{2\beta^2}$. Hence, 
$$
\Var_{P_XP_Y} [\alpha r_{\alpha, \beta, \gamma}(x,y) +  \frac{\gamma}{2} r^2_{\alpha, \beta, \gamma}(x,y)] \leq {\rm max}\bigg\{\Big(\frac{\alpha^2 }{2 \gamma}\Big)^2, \Big( \frac{2\alpha \beta + \gamma}{2\beta^2} \Big)^2\bigg\}.$$
Combining everything together,  we restate the Proposition 2 in the main text:
\begin{prop}[Asymptotic Variance of $\hat{J}^{m,n}_{\rm RPC}$] \label{prop:variance_asym} 
\begin{equation*}
\begin{split}
    \Var [\hat{J}^{m,n}_{{\rm RPC}}] & = \frac{1}{n} \Var_{P_{XY}} [r_{\alpha, \beta, \gamma}(x,y) - \frac{\beta}{2} r^2_{\alpha, \beta, \gamma}(x,y)] + \frac{1}{m} \Var_{P_XP_Y} [\alpha r_{\alpha, \beta, \gamma}(x,y) +  \frac{\gamma}{2} r^2_{\alpha, \beta, \gamma}(x,y)] + o(\frac{1}{n'}) \\
    &\leq \frac{1}{n} {\rm max}\bigg\{\Big(\frac{2\alpha \gamma + \beta \alpha^2}{2 \gamma^2}\Big)^2, \Big(\frac{1}{2\beta}\Big)^2\bigg\} + \frac{1}{m} {\rm max}\bigg\{\Big(\frac{\alpha^2 }{2 \gamma}\Big)^2, \Big( \frac{2\alpha \beta + \gamma}{2\beta^2} \Big)^2\bigg\} + o(\frac{1}{n'})
\end{split}
\end{equation*}
\label{prop:var3}
\end{prop}

\subsection{Proof of Proposition \ref{prop:variance_informal} in the Main Text - From Boundness of $f_\theta$}

As discussed in Assumption~\ref{assump:bound_f}, for the estimation $\hat{J}^{m,n}_{\rm RPC}$, we can bound the function $f_\theta$ in $\mathcal{F}_\Theta$ within $[-\frac{\alpha}{\gamma}, \frac{1}{\beta}]$ without losing precision. Then, re-arranging $\hat{J}^{m,n}_{\rm RPC}$:
\begin{equation*}
\begin{split}
    & \sup_{{f}_\theta \in \mathcal{F}_\Theta} \frac{1}{n}\sum_{i=1}^n f_\theta(x_i,y_i) -  \frac{1}{m}\sum_{j=1}^m \alpha f_\theta(x'_j,y'_j) - \frac{1}{n}\sum_{i=1}^n  \frac{\beta}{2} f_\theta^2(x_i,y_i) -  \frac{1}{m}\sum_{j=1}^m \frac{\gamma}{2} f_\theta^2(x'_j,y'_j) \\
    & \sup_{{f}_\theta \in \mathcal{F}_\Theta}\frac{1}{n}\sum_{i=1}^n \Big( f_\theta(x_i,y_i) - \frac{\beta}{2} f_\theta^2(x_i,y_i) \Big) + \frac{1}{m}\sum_{j=m}^n \Big( \alpha f_\theta(x'_j,y'_j) + \frac{\gamma}{2} f_\theta^2(x'_j,y'_j) \Big)
\end{split}
\end{equation*}
Then, since $-\frac{\alpha}{\gamma} \leq f_\theta(\cdot,\cdot) \leq  \frac{1}{\beta}$, basic calculations give us 
$$ -\frac{2\alpha \gamma +\beta \alpha^2}{2\gamma^2} \leq f_\theta(x_i,y_i) - \frac{\beta}{2} f_\theta^2(x_i,y_i) \leq \frac{1}{2\beta} \,\, {\rm and} \,\, -\frac{\alpha^2}{2\gamma}\leq \alpha f_\theta(x'_j,y'_j) + \frac{\gamma}{2} f_\theta^2(x'_j,y'_j) \leq \frac{2\alpha\beta + \gamma}{2\beta^2}.$$ 
The resulting variances have
$$
\Var[f_\theta(x_i,y_i) - \frac{\beta}{2} f_\theta^2(x_i,y_i)]\leq {\rm max}\,\bigg\{\Big(\frac{2\alpha \gamma +\beta \alpha^2}{2\gamma^2}\Big)^2, \Big(\frac{1}{2\beta}\Big)^2\bigg\}$$
and
$$\Var[\alpha f_\theta(x'_j,y'_j) + \frac{\gamma}{2} f_\theta^2(x'_j,y'_j)] \leq {\rm max}\,\bigg\{ \Big(\frac{\alpha^2}{2\gamma}\Big)^2, \Big(\frac{2\alpha\beta + \gamma}{2\beta^2}\Big)^2\bigg\}.
$$
Taking the mean of $m,n$ independent random variables gives the result:
\begin{prop}[Variance of $\hat{J}^{m,n}_{\rm RPC}$] \label{prop:variance} 
$$\Var [\hat{J}^{m,n}_{{\rm RPC}}] \leq \frac{1}{n} {\rm max}\bigg\{\Big(\frac{2\alpha \gamma + \beta \alpha^2}{2 \gamma^2}\Big)^2, \Big(\frac{1}{2\beta}\Big)^2\bigg\} + \frac{1}{m} {\rm max}\bigg\{\Big(\frac{\alpha^2 }{2 \gamma}\Big)^2, \Big( \frac{2\alpha \beta + \gamma}{2\beta^2} \Big)^2\bigg\}.
$$
\label{prop:var2}
\end{prop}

\subsection{Implementation of Experiments}
For visual representation learning, we follow the implementation in~\url{https://github.com/google-research/simclr}. For speech representation learning, we follow the implementation in~\url{https://github.com/facebookresearch/CPC_audio}. For MI estimation, we follow the implementation in~\url{https://github.com/yaohungt/Pointwise_Dependency_Neural_Estimation/tree/master/MI_Est_and_CrossModal}..

\subsection{Relative Predictive Coding on Vision} \label{appendix:vision}

The whole pipeline of pretraining contains the following steps: First, a stochastic data augmentation will transform one image sample $\vx_k$ to two different but correlated augmented views, $\vx_{2k-1}'$ and $\vx_{2k}'$. Then a base encoder $f(\cdot)$ implemented using ResNet \citep{he2016deep} will extract representations from augmented views, creating representations $\vh_{2k-1}$ and $\vh_{2k}$. Later a small neural network $g(\cdot)$ called projection head will map  $\vh_{2k-1}$ and $\vh_{2k}$ to $\vz_{2k-1}$ and $\vz_{2k}$ in a different latent space. For each minibatch of $N$ samples, there will be $2N$ views generated. For each image $\vx_{k}$ there will be one positive pair $\vx_{2k-1}'$ and $\vx_{2k}'$ and $2(N-1)$ negative samples. The RPC loss between a pair of positive views, $\vx_{i}'$ and $\vx_{j}'$ (augmented from the same image) , can be calculated by the substitution $f_{\theta}(\vx_{i}', \vx_{j}') = (\vz_i \cdot \vz_j) / \tau = s_{i,j}$ ($\tau$ is a hyperparameter) to the definition of RPC: 
\begin{equation}  
    \ell_{i, j}^{\mathrm{RPC}}= 
    -(s_{i,j}
    -  \frac{\alpha}{2(N-1)} \sum_{k=1}^{2 N} \1_\mathrm{[k \neq i]} s_{i,k}
    -  \frac{\beta}{2} s_{i,j}^2
    -  \frac{\gamma}{2 \cdot 2(N-1)} \sum_{k=1}^{2 N} \1_\mathrm{[k \neq i]} s_{i,k}^2)
\end{equation}

For losses other than RPC, a hidden normalization of $s_{i,j}$ is often required by replacing $\vz_i \cdot \vz_j$ with $(\vz_i \cdot \vz_j) / |\vz_i||\vz_j|$. CPC and WPC adopt this, while other objectives needs it to help stabilize training variance. RPC does not need this normalization.

\subsection{CIFAR-10/-100 and ImageNet Experiments Details}

\paragraph{ImageNet} Following the settings in \citep{chen2020simple, chen2020big}, we train the model on Cloud TPU with $128$ cores, with a batch size of $4,096$ and global batch normalization \footnote{For WPC \citep{ozair2019wasserstein}, the global batch normalization during pretraining is disabled since we enforce 1-Lipschitz by gradient penalty \citep{gulrajani2017improved}.} \citep{ioffe2015batch}. Here we refer to the term batch size as the number of images (or utterances in the speech experiments) we use per GPU, while the term minibatch size refers to the number of negative samples used to calculate the objective, such as CPC or our proposed RPC. The largest model we train is a 152-layer ResNet with selective kernels (SK) \citep{li2019selective} and $2\times$ wider channels. We use the LARS optimizer \citep{you2017large} with momentum $0.9$. The learning rate linearly increases for the first $20$ epochs, reaching a maximum of $6.4$, then decayed with cosine decay schedule. The weight decay is $10^{-4}$. A MLP projection head $g(\cdot)$ with three layers is used on top of the ResNet encoder. Unlike \citet{chen2020big}, we do not use a memory buffer, and train the model for only $100$ epochs rather than $800$ epochs due to computational constraints. These two options slightly reduce CPC's performance benchmark for about $2\%$ with the exact same setting. The unsupervised pre-training is followed by a supervised fine-tuning. Following SimCLRv2 \citep{chen2020simple,chen2020big}, we fine-tune the 3-layer $g(\cdot)$ for the downstream tasks. We use learning rates $0.16$ and $0.064$ for standard 50-layer ResNet and larger 152-layer ResNet respectively, and weight decay and learning rate warmup are removed. Different from \citet{chen2020big}, we use a batch size of $4,096$, and we do not use global batch normalization for fine-tuning. For $J_{\rm RPC}$ we disable hidden normalization and use a temperature $\tau = 32$. For all other objectives, we use hidden normalization and $\tau = 0.1$ following previous work \citep{chen2020big}. For relative parameters, we use $\alpha=0.3, \beta=0.001, \gamma=0.1$ and $\alpha=0.3, \beta=0.001, \gamma=0.005$ for ResNet-50 and ResNet-152 respectively. 

\paragraph{CIFAR-10/-100} Following the settings in \citep{chen2020simple}, we train the model on a single GPU, with a batch size of $512$ and global batch normalization \citep{ioffe2015batch}. We use ResNet \citep{he2016deep} of depth $18$ and depth $50$, and does not use Selective Kernel \citep{li2019selective} or a multiplied width size. We use the LARS optimizer \citep{you2017large} with momentum $0.9$. The learning rate linearly increases for the first $20$ epochs, reaching a maximum of $6.4$, then decayed with cosine decay schedule. The weight decay is $10^{-4}$. A MLP projection head $g(\cdot)$ with three layers is used on top of the ResNet encoder. Unlike \citet{chen2020big}, we do not use a memory buffer.  We train the model for $1000$ epochs. The unsupervised pre-training is followed by a supervised fine-tuning. Following SimCLRv2 \citep{chen2020simple,chen2020big}, we fine-tune the 3-layer $g(\cdot)$ for the downstream tasks. We use learning rates $0.16$ for standard 50-layer ResNet , and weight decay and learning rate warmup are removed. For $J_{\rm RPC}$ we disable hidden normalization and use a temperature $\tau = 128$. For all other objectives, we use hidden normalization and $\tau = 0.5$ following previous work \citep{chen2020big}. For relative parameters, we use $\alpha = 1.0, \beta = 0.005, \text{ and } \gamma = 1.0$. 

\paragraph{STL-10} We also perform the pre-training and fine-tuning on  STL-10 \citep{coates2011analysis} using the model proposed in \citet{chuang2020debiased}. \citet{chuang2020debiased} proposed to indirectly approximate the distribution of negative samples so that the objective is \textit{debiased}. However, their implementation of contrastive learning is consistent with \citet{chen2020simple}. We use a ResNet with depth $50$ as an encoder for pre-training, with Adam optimizer, learning rate $0.001$ and weight decay $10^{-6}$. The temperature $\tau$ is set to $0.5$ for all objectives other than $J_{\rm RPC}$, which disables hidden normalization and use $\tau=128$. The downstream task performance increases from $83.4\%$ of $J_{\rm CPC}$ to $84.1\%$ of $J_{\rm RPC}$.

\begin{table}
\centering
\begin{tabular}{ |p{3cm}||p{3cm}|p{3cm}|p{3cm}| }
 \hline
 \multicolumn{4}{|c|}{Confidence Interval of $J_{\rm RPC}$ and $J_{\rm CPC}$} \\
 \hline
 Objective & CIFAR 10 & CIFAR 100 & ImageNet \\
 \hline
 $J_{\rm CPC}$ & $(91.09\%, 91.13\%)$ & $(77.11\%, 77.36\%)$ & $(73.39\%, 73.48\%)$ \\
 $J_{\rm RPC}$ & $(91.16\%, 91.47\%)$ & $(77.41\%, 77.98\%)$ & $(73.92\%, 74.43\%)$ \\
 \hline
\end{tabular}
\label{table:ci}
\caption{Confidence Intervals of performances of $J_{\rm RPC}$ and $J_{\rm CPC}$ on CIFAR-10/-100 and ImageNet.}
\end{table}

\paragraph{Confidence Interval} We also provide the confidence interval of $J_{\rm RPC}$ and $J_{\rm CPC}$ on CIFAR-10, CIFAR-100 and ImageNet, using ResNet-18, ResNet-18 and ResNet-50 respectively (95\% confidence level is chosen) in Table 4. Both CPC and RPC use the same experimental settings throughout this paper. Here we use the relative parameters ($\alpha = 1.0, \beta = 0.005, \gamma=1.0$) in $J_{\rm RPC}$ which gives the best performance on CIFAR-10. The confidence intervals of CPC do not overlap with the confidence intervals of RPC, which means the difference of the downstream task performance between RPC and CPC is statistically significant.

\subsection{Relative Predictive Coding on Speech} \label{appendix:speech}
% \colorbox{yellow}{remove latent}

For speech representation learning, we adopt the general architecture from \citet{oord2018representation}. Given an input signal $\vx_{1:T}$ with $T$ time steps, we first pass it through an encoder $\phi_{\theta}$ parametrized by $\theta$ to produce a sequence of hidden representations $\{\vh_{1:T}\}$ where $\vh_t = \phi_{\theta}(\vx_t)$. After that, we obtain the contextual representation $\vc_t$ at time step $t$ with a sequential model $\psi_{\rho}$ parametrized by $\rho$: $\mathbf{c}_{t}=\psi_{\rho}(\vh_1, \ldots, \vh_t)$, where $\vc_t$ contains context information before time step $t$. For unsupervised pre-training, we use a multi-layer convolutional network as the encoder $\phi_{\theta}$, and an LSTM with hidden dimension $256$ as the sequential model $\psi_{\rho}$. Here, the contrastiveness is between the positive pair $(\vh_{t+k}, \vc_t)$ where $k$ is the number of time steps ahead, and the negative pairs $(\vh_i, \vc_t)$, where $\vh_i$ is randomly sampled from $\mathcal{N}$, a batch of hidden representation of signals assumed to be unrelated to $\vc_t$. The scoring function $f$ based on Equation \ref{eq:RPC_math_def} at step $t$ and look-ahead $k$ will be $f_k = f_k(\vh, \vc_t) = \exp((\vh)^\top \mW_k \vc_t)$, where $\mW_k$ is a learnable linear transformation defined separately for each $k\in\{1,...,K\}$ and $K$ is predetermined as $12$ time steps. The loss in \Eqref{eq:RPC_math_def} will then be formulated as:
\begin{equation} \label{eq:RPC_speech}
    \ell^{\mathrm{RPC}}_{t,k} = - (f_k(\vh_{t+k}, \vc_t) - \frac{\alpha}{|\mathcal{N}|}\sum_{\vh_i\in \mathcal{N}}f_k(\vh_i, \vc_t) - \frac{\beta}{2}f_k^2(\vh_{t+k}, \vc_t) - \frac{\gamma}{2 |\mathcal{N}|}\sum_{\vh_i\in \mathcal{N}}f_k^2(\vh_i, \vc_t))
\end{equation}

We use the following relative parameters: $\alpha=1, \beta=0.25, \text{ and }\gamma=1$, and we use the temperature $\tau=16$ for $J_{\rm RPC}$. For $J_{\rm CPC}$ we follow the original implementation which sets $\tau=1$. We fix all other experimental setups, including architecture, learning rate, and optimizer. As shown in Table \ref{tab:speech_res}, $J_{\rm RPC}$ has better downstream task performance, and is closer to the performance from a fully supervised model.

\subsection{Empirical Observations on Variance and Minibatch Size}
\label{appendix:var_bsz}
\paragraph{Variance Experiment Setup}
We perform the variance comparison of $J_{\rm DV}$, $J_{\rm NWJ}$ and the proposed $J_{\rm RPC}$. The empirical experiments are performed using SimCLRv2 \citep{chen2020big} on CIFAR-10 dataset. We use a ResNet of depth $18$, with batch size of $512$. We train each objective with $30$K training steps and record their value. In Figure \ref{fig:var_batch}, we use a temperature $\tau=128$ for all objectives. Unlike other experiments, where hidden normalization is applied to other objectives, we remove hidden normarlization for all objectives due to the reality that objectives after normalization does not reflect their original values. From Figure \ref{fig:var_batch}, $J_{\rm RPC}$ enjoys lower variance and more stable training compared to  $J_{\rm DV}$ and $J_{\rm NWJ}$.

\paragraph{Minibatch Size Experimental Setup}
We perform experiments on the effect of batch size on downstream performances for different objective. The experiments are performed using SimCLRv2 \citep{chen2020big} on CIFAR-10 dataset, as well as the model from \citet{riviere2020unsupervised} on LibriSpeech-100h dataset \citep{panayotov2015librispeech}. For vision task, we use the default temperature $\tau=0.5$ from \citet{chen2020big} and hidden normalization mentioned in Section \ref{sec:experiment} for $J_{\rm CPC}$. For $J_{\rm RPC}$ in vision and speech tasks we use a temperature of $\tau=128$ and $\tau=16$ respectively, both without hidden normalization.

\subsection{Mutual Information Estimation} \label{appendix:mi}

Our method is compared with baseline methods CPC~\citep{oord2018representation}, NWJ~\citep{nguyen2010estimating}, JSD~\citep{nowozin2016f}, and SMILE~\citep{song2019understanding}. All the approaches consider the same design of $f(x,y)$, which is a 3-layer neural network taking concatenated $(x,y)$ as the input. We also fix the learning rate, the optimizer, and the minibatch size across all the estimators for a fair comparison. 

We present results of mutual information by Relative Predictive Coding using different sets of relative parameters in Figure \ref{fig:MI_rpc}. In the first row, we set $\beta = 10^{-3}$, $\gamma = 1$, and experiment with different $\alpha$ values. In the second row, we set $\alpha = 1$, $\gamma = 1$ and in the last row we set $\alpha=1$, $\beta=10^{-3}$. From the figure, a small $\beta$ around $10^{-3}$ and a large $\gamma$ around $1.0$ is crucial for an estimation that is relatively low bias and low variance. This conclusion is consistent with Section 3 in the main text. 

\begin{figure}[t!]
\begin{center}
%\framebox[4.0in]{$\;$}
% \includegraphics[scale=0.27]{fig/gaussian.png}
\includegraphics[scale=0.27]{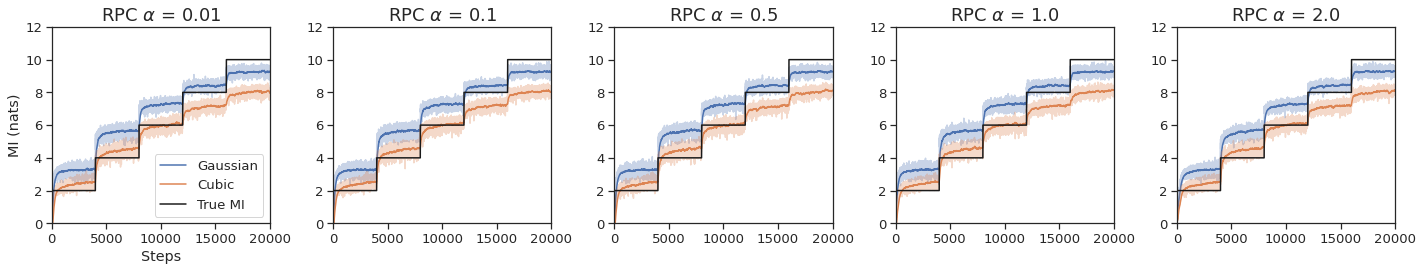}
\includegraphics[scale=0.27]{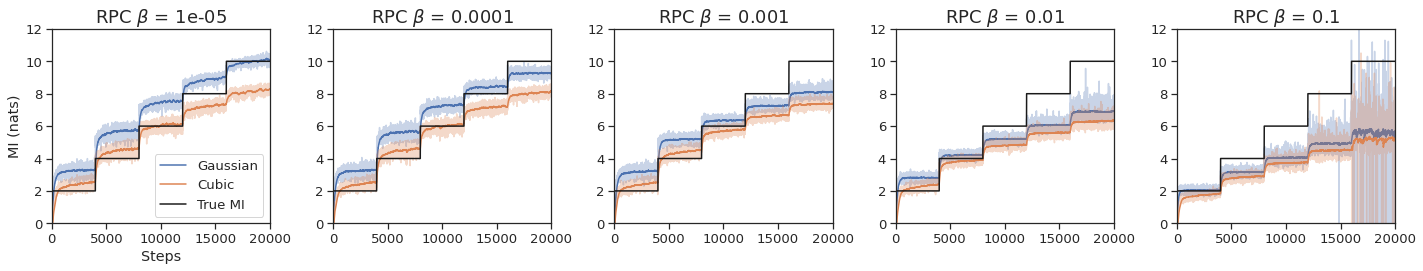}
\includegraphics[scale=0.27]{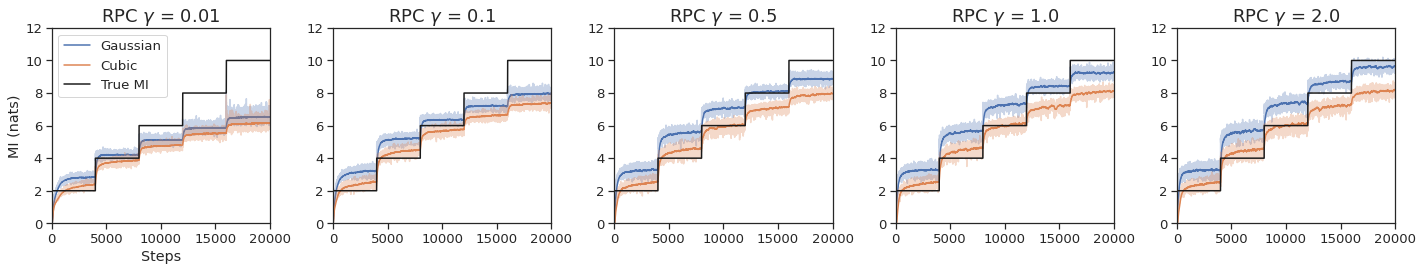}
\end{center}
\vspace{-5mm}
\caption{Mutual information estimation by RPC performed on 20-d correlated Gaussian distribution, with different sets of relative parameters.
}
\label{fig:MI_rpc}
\end{figure}

We also performed comparison between $J_{\rm RPC}$ and Difference of Entropies (DoE)~\citep{mcallester2020formal}. We performed two sets of experiments: in the first set of experiments we compare $J_{\rm RPC}$ and DoE when MI is large ($>100$ nats), while in the second set of experiments we compare $J_{\rm RPC}$ and DoE using the setup in this section (MI $< 12$ nats and MI increases by 2 per 4k training steps). On the one hand, when MI is large ($> 100$ nats), we acknowledge that DoE is performing well on MI estimation, compared to $J_{\rm RPC}$ which only estimates the MI around $20$. This analysis is based on the code from \url{https://github.com/karlstratos/doe}. On the other hand, when the true MI is small, the DoE method is more unstable than $J_{\rm RPC}$, as shown in Figure \ref{fig:MI_doe}. Figure \ref{fig:MI_doe} illustrates the results of the DoE method when the distribution is isotropic Gaussian (correctly specified) or Logistic (mis-specified). Figure~\ref{fig:MI} only shows the results using Gaussian.

\begin{figure}[t!]
\begin{center}
%\framebox[4.0in]{$\;$}
% \includegraphics[scale=0.27]{fig/gaussian.png}
\includegraphics[scale=0.25]{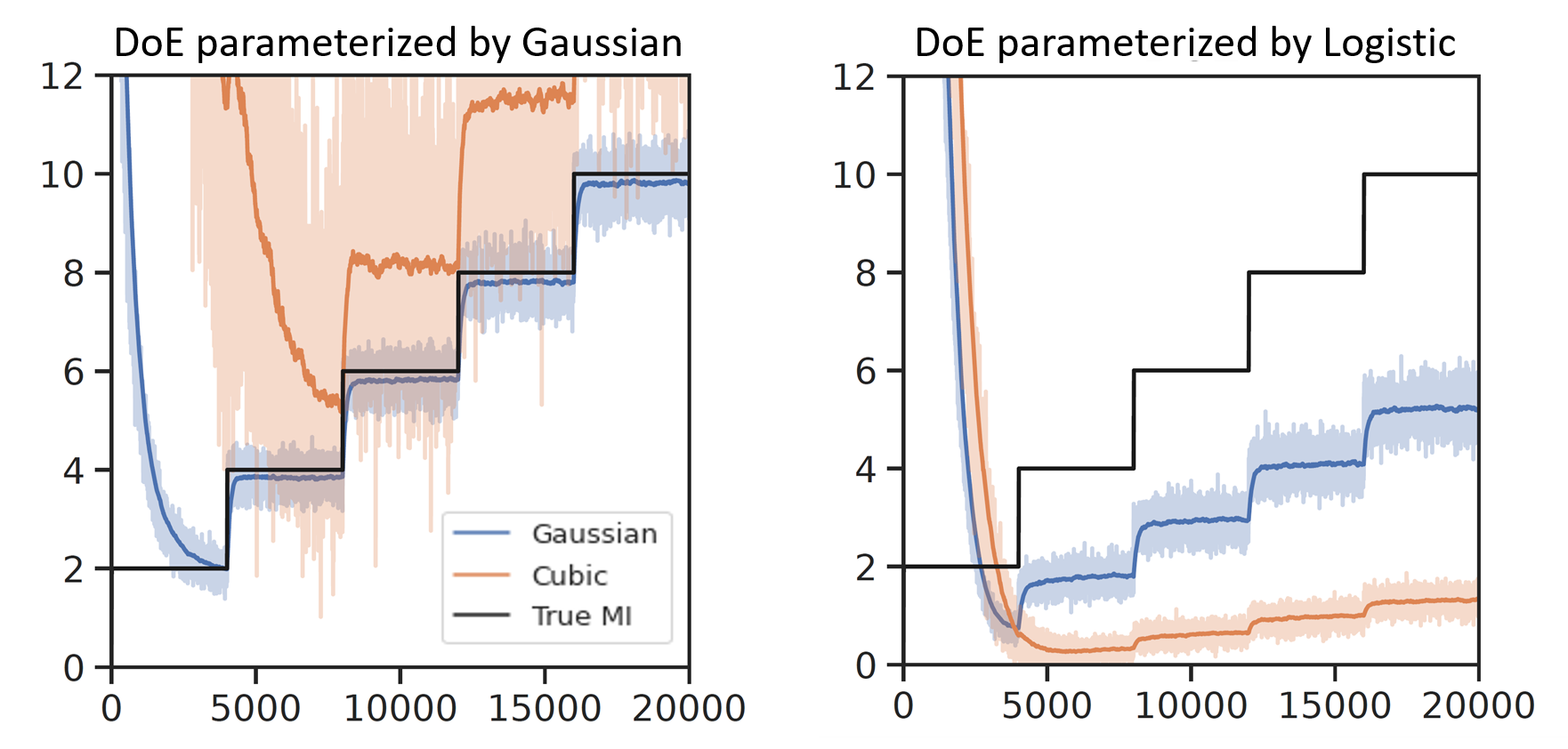}
\end{center}
\vspace{-5mm}
\caption{Mutual information estimation by DoE performed on 20-d correlated Gaussian distribution. The figure on the left shows parametrization under Gaussian (correctly specified), and the figure on the right shows parametrization under Logistic (mis-specified).}

\label{fig:MI_doe}
\end{figure}

\end{document}